\documentclass[11pt, a4paper]{article}
\usepackage{amsmath,amsfonts,amsthm,mathrsfs}

\pdfoutput=1

\numberwithin{equation}{section}
\newtheorem{theorem}{Theorem}
\newtheorem{lemma}{Lemma}
\newtheorem{proposition}{Proposition}

\newtheorem{remark}{Remark}

\usepackage[scale={0.72,0.743}, hmarginratio=1:1, vmarginratio=1:1, headheight=2ex, headsep = 2ex, footskip = 3.9ex, nomarginpar]{geometry}

\usepackage{booktabs}
\usepackage{diagbox} 
\usepackage{multirow} 
\usepackage{makecell} 
\usepackage{longtable} 
\usepackage{array}
\usepackage{tabularx}
\usepackage{caption}	
\usepackage{graphicx}
\usepackage{subfigure} 
\usepackage{hyperref}
\usepackage{cleveref}
\usepackage{bm}
\usepackage{bbm}

\def\NN{\mathbb N}

\def\RR{\mathbb R}

\newtheorem{assumption}{Assumption}

\DeclareMathOperator*{\argmin}{arg\,min}

\begin{document}
	
		\title{Can overfitted deep neural networks in adversarial training generalize? - An approximation viewpoint }
	\author{Zhongjie Shi\\ \footnotesize Department of Statistics and Actuarial Science, The University of Hong Kong,  \\ \footnotesize
		Pok Fu Lam Road, Hong Kong, Email: zshi2@hku.hk
		\and
		Fanghui Liu\\ \footnotesize Department of Computer Science, University of Warwick, \\
		\footnotesize  Coventry, United Kingdom, Email: fanghui.liu@warwick.ac.uk
		\and
		Yuan Cao\\ \footnotesize Department of Statistics and Actuarial Science, The University of Hong Kong, \\
		\footnotesize Pok Fu Lam Road, Hong Kong, Email: yuancao@hku.hk
		\and
		Johan A.K. Suykens \\ \footnotesize Department of Electrical Engineering, ESAT-STADIUS, KU Leuven, \\ \footnotesize
		Kasteelpark Arenberg 10, B-3001 Leuven, Belgium, Email: johan.suykens@esat.kuleuven.be}
	
	\date{}
	
	\maketitle

\begin{abstract}
Adversarial training is a widely used method to improve the robustness of deep neural networks (DNNs) over adversarial perturbations.
However, it is empirically observed that adversarial training on over-parameterized networks often suffers from the \textit{robust overfitting}: it can achieve almost zero adversarial training error while the robust generalization performance is not promising.
In this paper, we provide a theoretical understanding of the question of whether overfitted DNNs in adversarial training can generalize from an approximation viewpoint. Specifically, our main results are summarized into three folds:
i) For classification, we prove by construction the existence of infinitely many adversarial training classifiers on over-parameterized DNNs that obtain arbitrarily small adversarial training error (overfitting), whereas achieving good robust generalization error under certain conditions concerning the data quality, well separated, and perturbation level. 
ii) Linear over-parameterization (meaning that the number of parameters is only slightly larger than the sample size) is enough to ensure such existence if the target function is smooth enough.
iii) For regression, our results demonstrate that there also exist infinitely many overfitted DNNs with linear over-parameterization in adversarial training that can achieve almost optimal rates of convergence for the standard generalization error. 
Overall, our analysis points out that robust overfitting can be avoided but the required model capacity will depend on the smoothness of the target function, while a robust generalization gap is inevitable. We hope our analysis will give a better understanding of the mathematical foundations of robustness in DNNs from an approximation view.

\end{abstract}
	
	
	
	\section{Introduction}
	
 Deep neural networks (DNNs) have achieved great empirical success but are demonstrated to be susceptible to small perturbations~\cite{Goodfellow2016}.
 To be specific, under adversarially chosen, albeit imperceptible, perturbations to their inputs, a.k.a., \textit{adversarial examples}, a well-performed DNN achieves quite low accuracy on these adversarial examples~\cite{szegedy2013intriguing,goodfellow2014explaining}.
 This results in a high risk of building robust, secure, trustworthy machine learning systems.
 To improve the robustness of DNNs, a series of methods are proposed to defend against artificially designed adversarial attacks aiming at fooling the models \cite{gu2014towards,carlini2017adversarial,madry2017towards,zhang2019theoretically,delgosha2022robust,bai2023efficient}. 
 Among them, \textit{adversarial training} \cite{madry2017towards} is one of the most empirically successful methods to defend against adversarial examples via a min-max optimization. 
	
 Mathematically, let $\mathcal{X} \subset \RR^d$ be the input space, $Y \subset \RR$ be the output space, and  $\mathcal{F}$ be the hypothesis space, e.g., the class of DNNs. Suppose that the data set $\{(\bm x_i,y_i)\}_{i=1}^n$ are i.i.d. sampled from the true unknown Borel probability distribution $\rho$ on $\mathcal{Z}= \mathcal{X} \times Y$. Then adversarial training aims to solve the following empirical (adversarial) risk minimization under a certain $\ell_\infty$ white-box adversarial attack
	\begin{equation} \label{advtraining}
		\min_{f\in \mathcal{F}}  \frac{1}{n} \sum_{i=1}^n \max_{\bm x'_i \in B_{\delta,\infty}(\bm x_i)} \ell \left(f(\bm x'_i),y_i \right), 
	\end{equation}
	where $\ell: \RR \times \RR \rightarrow \RR$ is the loss function which evaluates the cost between the model output and the corresponding label, $B_{\delta,\infty}(\bm z)= \{\bm x: \|\bm x-\bm z\|_\infty \leq \delta\} \cap \mathcal{X}$ is the $\delta$-ball (i.e., the perturbation radius $\delta$) centered at $\bm z$ w.r.t. $\ell_\infty$ norm.


Taking $\delta = 0$ in \Cref{advtraining}, adversarial training degenerates to standard training. Empirical and theoretical studies \cite{zhang2021understanding,bartlett2020benign,cao2022benign,tsigler2023benign,zhu2023benign} indicate that DNNs in the over-parameterized regime (i.e., the number of parameters is much larger than the training data size) can achieve zero training error under noisy data, but still generalize well.
This is called the \emph{benign overfitting} phenomenon\footnote{In this paper, \emph{benign overfitting} is given with a broader meaning, i.e., achieving almost zero training error as well as good generalization performance.}.
When it comes to adversarial training \eqref{advtraining} with $\delta > 0$, empirical observations \cite{su2018robustness,raghunathan2019adversarial,rice2020overfitting} demonstrate the \textit{robust overfitting} phenomenon in the over-parameterized regime of adversarial training instead, i.e., overfitting to the training set (achieving small adversarial training error or called train robust error) does harm the \textit {robust generalization} to a large extent for multiple datasets.
Moreover, there exists a large \textit{robust generalization gap} between the robust generalization and the standard generalization performance \cite{madry2017towards,rice2020overfitting}.  For example, as shown in \Cref{Figure1}
 from \cite{rice2020overfitting}, the adversarial training can achieve almost zero train robust error, but the test robust error is only nearly 50\%, and is much higher than the test standard error (slightly smaller than 20\%).

    \begin{figure}[htbp]
	\centering
		\includegraphics[width=0.5\textwidth]{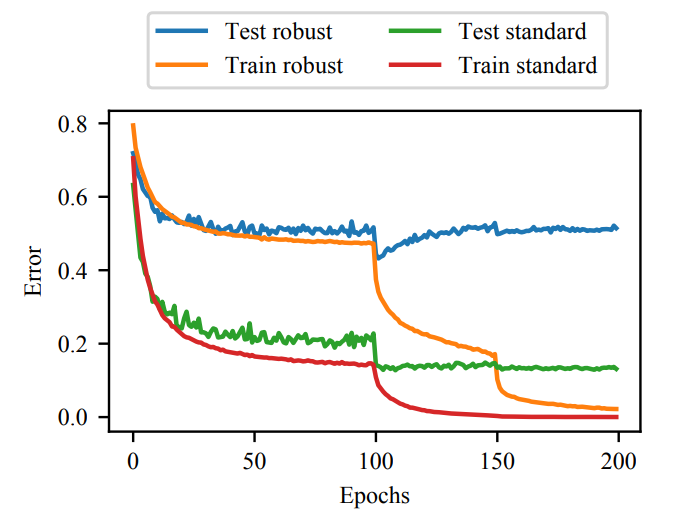}
		\caption{The learning curves of adversarial training on CIFAR-10 with $\delta=8/255$ \cite{rice2020overfitting}, while CIFAR-10 is 0.212-separated \cite{yang2020closer}. }
		\label{Figure1}
    \end{figure}
 
 Recent work \cite{yang2020closer} finds that real datasets have a natural separation property which is called \emph{$\epsilon$-separated}: input data points from different classes have at least 2$\epsilon$ distance in the pixel space. 
 For example, on the CIFAR-10 dataset, $\epsilon = 0.212$ \cite{yang2020closer}, which is much larger than the commonly used attack level $\delta=8/255$.
Due to this well-separated property, ideally, there exist DNNs that can achieve both good robustness and accuracy simultaneously.
Nevertheless, this target makes the parameter size of DNNs suffer from the curse of dimensionality as suggested by \cite{li2022robust}, i.e., an $\exp(\Omega(d))$ lower bound on the network size is inevitable.
At first glance, this result is counter-intuitive due to the following two reasons:

\begin{itemize}
    \item Due to the nice separation property, if the target function is sufficiently smooth or possesses some special structure, the required parameter size of DNNs is not needed in the exponential order of the input dimension $d$ from the perspective of approximation.
    \item Adversarial training degenerates to the standard training when taking the perturbation radius $\delta=0$, as shown in \Cref{advtraining}. If $\delta$ is sufficiently small, under well-behaved data distribution, robust overfitting can be avoided and benign overfitting can naturally arise without \emph{curse of dimensionality}, as empirically suggested by \cite{dong2021data,dong2022exploring,dong2022label}.
\end{itemize}

The above two reasons motivate us to carefully rethink the following question:
\begin{center}
    \emph{Can overfitted DNNs in adversarial training generalize with reasonable model complexity?}
\end{center}
\vspace{0.3cm}
We give an affirmative answer to this question from an approximation viewpoint by providing a comprehensive analysis to close the gap between theory and practice as much as possible.
In our analysis, we consider the data quality, the regularity condition of the target function, and the existence of label noise, and check the existence of the adversarial training global minima with good robust generalization performance. 
We make the following contributions and findings in this paper:

\begin{itemize}
    \item We prove by construction that there exist infinitely many over-parameterized DNNs that can achieve zero adversarial training error as well as good robust generalization error of the same order as the lower bound. Such construction is based on the condition on data and perturbation, i.e., the data distribution is of relatively high quality and is well-separated; the perturbation radius of adversarial training is small enough.
    Furthermore, if the target function is smooth enough, DNNs with $\Omega(n)$ parameters (i.e., linear over-parameterization) are sufficient to achieve both robustness and accuracy simultaneously.
    \item Our construction of the adversarial training global minima is almost \emph{optimal} since its robust generalization upper bound matches the order of the lower bound. We also theoretically demonstrate the existence of the robust generalization gap by showing that even for these adversarial training classifiers with good robust performance, their robust generalization error is still worse than their standard generalization error.
\end{itemize}
Accordingly, our theoretical analysis provides an in-depth understanding of overfitting in adversarial training on over-parameterized DNNs from the perspective of approximation, and provides a relaxed model complexity requirement as well as the best possible robust generalization under this circumstance.
Note that our construction is data (distribution)-dependent, which allows us to study the \emph{limit} performance of DNNs under adversarial training, and hence our analysis points out that \emph{robust generalization gap} is inevitable. 
We expect that our results will be beneficial to the dynamics analysis of DNNs under adversarial training algorithms. 
	
The rest of the paper is organized as follows. In \Cref{relatedwork}, we give an overview of related works close to our paper. The problem setting and common assumptions for classification and regression is introduced in \Cref{sec:setting}. Then in \Cref{sectionclass}, we present our main results about robust and standard generalization performance of good adversarial training global minima for the classification tasks. In \Cref{section3}, we extend our results in \Cref{sectionclass} to the regression tasks. \Cref{conclusion} draws a conclusion of this paper. 
The proofs of our theoretical results can be found in the Appendix.

	\section{Related Work} \label{relatedwork}

 	It is empirically observed in previous works that adversarial training in over-parameterized regime sometimes might overfit, i.e., the robust overfitting phenomenon \cite{su2018robustness,raghunathan2019adversarial,rice2020overfitting}. However, the mechanism for this is still unclear. Many works try to figure out the important elements in adversarial training that might lead to robust generalization. \cite{schmidt2018adversarially,bhagoji2019lower,dan2020sharp} manifest that to achieve robust generalization in adversarial training, it requires a larger sample complexity compared with standard training.
  These results also demonstrate that data distribution is a vital factor in adversarial training. \cite{ilyas2019adversarial} indicates that non-robust features in the data can hurt robust generalization, thus making the data quality significant in adversarial training. \cite{shafahi2018are,ding2019sensitivity} also show that robust generalization obtained by adversarial training essentially hinges on the property of the data distribution. \cite{dong2021data,dong2022label} further demonstrate that adversarial training can achieve better robust generalization by utilizing data samples with higher quality.
  
	Some works have studied the robust generalization of adversarial training in the under-parameterized regime \cite{khim2018adversarial,yin2019rademacher,xiao2022adversarial,muthukumar2023adversarial}. For example, \cite{yin2019rademacher} presents the adversarial generalization error bounds by adversarial Rademacher complexity, and \cite{xiao2022adversarial} estimates the bounds of adversarial Rademacher complexity of deep neural networks. However, such adversarial generalization error bounds only work for adversarial training in the under-parameterized regime and are not suitable for the over-parameterized regime.  It is still unknown whether benign overfitting exists in adversarial training on over-parameterized DNNs, although there are some attempts in some simple settings. For instance,  \cite{chen2021benign} demonstrates the occurrence of benign overfitting in the adversarially robust linear classification with sub-Gaussian mixture data for both standard and robust generalization. While \cite{li2023clean} shows that standard generalization and robust overfitting both happen in adversarial training for the patch data distribution.
	
Besides, there exist some works close to the scope of this paper that tries to understand the robust generalization from an approximation theory viewpoint \cite{liu2022benefits,li2022robust}. In \cite{liu2022benefits}, they study the difference between the robust generalization and standard generalization for the true regression function. In \cite{li2022robust}, they demonstrate that the network requires $\mathcal{O}(\epsilon^{-d})$ parameters to achieve zero robust generalization for all the $2\epsilon$-separated data distribution when there is no noise. However, the networks they constructed to achieve such robust generalization are irrelevant to the networks learned from adversarial training with the usage of data samples, making their analysis not suitable for the study of the robust generalization for overfitting networks under adversarial training.


	
	
	


	
	

 \section{Problem settings and common assumptions} 
 \label{sec:setting}
 Here we introduce the common problem settings for classification and regression under adversarial training, e.g., the learning framework, the hypothesis space, and the model formulation.
The specific problem settings and assumptions for classification and regression can be found in \Cref{assumptionclass} and \Cref{section31}, respectively.

\noindent {\bf Learning framework:} We follow the classical statistical learning framework \cite{Cucker2007}. Suppose that the data sample $D=\{(\bm x_i,y_i)\}^{n}_{i=1} \subset \mathcal{Z}^n$ are i.i.d. sampled from a Borel probability measure $\rho$ on $\mathcal{Z}=\mathcal{X}\times Y$ with $\mathcal{X} \subset [0,1]^d$, $Y=\{-1,1\}$ for binary classification and $Y \subseteq[-M,M]$ with some $M>0$ for regression. 
 For notational simplicity, denoting $X:= \{\bm x_i\}_{i=1}^n$, the separation distance of the input data sample $X$ satisfies \cite{wendland2004scattered}
	\begin{equation}
		q_X:= \frac{1}{2} \min_{i \neq j} \|\bm x_i-\bm x_j\|_\infty \leq n^{-\frac{1}{d}} \,,
	\end{equation}
	which is half of the minimal distance between two distinct input data samples admitting an upper bound w.r.t. $n$ and $d$.
 
Denote $\rho(y|\bm x)$ as the conditional distribution at $\bm x\in\mathcal{X}$ induced by $\rho$, $\rho_X$ as the marginal distribution of $\rho$ on $\mathcal{X}$, and $(L^2_{\rho_X},\parallel\cdot\parallel_\rho)$ as the Hilbert space of square-integrable functions with respect to $\rho_X$.
 The objective of learning for classification or regression is to find a learning model that is a good approximation of the ``target function", which is defined as the conditional mean $f_\rho(\bm x)= \int_Y y d\rho(y|\bm x)$.
 Here the used learning model is a fully-connected deep neural network as described below.

\noindent {\bf Model formulation of DNNs:} Regarding the learning model, we consider standard fully-connected deep neural networks (FNNs) with the ReLU activation function in this paper. Denote the affine operator $\mathcal{A}_\ell: \RR^{d_{\ell-1}} \to  \RR^{d_{\ell}}$ as $\mathcal{A}_\ell (\bm x)= \bm W_\ell \bm x+ \bm b_\ell$, where $\bm W_\ell \in \RR^{d_\ell \times d_{\ell-1}}$ is the weight matrix and $\bm b_\ell \in \RR^{d_\ell}$ is the bias vector. A deep ReLU FNN with depth $L$ and width $\{d_\ell\}_{\ell=1}^L$ is defined as
	\begin{equation} \label{FNN}
		f(\bm x)= c \cdot \sigma \circ \mathcal{A}_L \circ \sigma \circ \mathcal{A}_{L-1} \circ \cdots  \circ \sigma \circ \mathcal{A}_1 (\bm x)\,,
	\end{equation}
	where $d_0=d$, $\bm c\in \RR^{d_L}$ is the coefficient vector, $\{\bm W_\ell\}_{\ell=1}^L$ are the weight matrices and $\{\bm b_\ell\}_{\ell=1}^L$ are the bias vectors. The number of parameters in this network is
	\begin{equation}
		\mathcal{N}= d_L + \sum_{\ell=1}^L (d_{\ell-1} d_\ell + d_\ell)\,.
	\end{equation}
	We denote the hypothesis space $\mathcal{F}_{\vec{d},L}$ as the collection of all deep ReLU FNNs with the form \eqref{FNN}.

\noindent {\bf Common assumptions:} In the next, we make two assumptions: one is about the distortion of $L_{\rho_X}$ with respect to the Lebesgue measure \cite{chui2020realization} and the other one is the regularity assumption of the ``target function".
  \begin{assumption}\cite[non-irregularity of $\rho_X$]{shi2013learning,chui2020realization} \label{assumption3}
    Let  $J_\rho$ be the identity mapping $J_\rho: L^1(\mathcal{X}) \to L_{\rho_X}^1(\mathcal{X})$ and $\|J_\rho\|$ be the operator norm. Similarly, we denote $\bar J_\rho$ as the identity mapping $\bar J_\rho: L_{\rho_X}^1(\mathcal{X}) \to L^1(\mathcal{X})$, and $\|\bar J_\rho\|$ as the corresponding operator norm.
    We assume that
        \begin{equation} \label{distortionass}
		  \|J_\rho\| < \infty, \quad 	\|\bar J_\rho\| < \infty.
	\end{equation}
 	Moreover, we denote $\Phi_\rho$ as the set of $\rho_X$ that satisfies  \Cref{assumption3}.
    \end{assumption}
 \begin{remark}
 This assumption on the marginal distribution $\rho_X$ is similar as \cite{shi2013learning,chui2020realization} to ensure that $\rho_X$ is not that irregular. 
 Admittedly, it is a little stronger than the standard assumption that $\rho_X$ is absolutely continuous with respect to the Lebesgue measure. However, this assumption can be satisfied when $\rho_X$ is some common distribution with bounded support, e.g., uniform distribution.
 \end{remark}

\noindent {\bf H\"older continuity:} We assume the ``target function" satisfies some smoothness level which is of interest for ease of analysis.
We describe it in H\"older spaces, i.e., the $\alpha$-H\"older continuous functions $W_\infty^\alpha (\mathcal{X})$ with $\alpha>0$ \cite{Yarotsky2017,SchmidtHieber2020}.
 To be specific, for  $\alpha \in (0,1]$, $W_\infty^\alpha (\mathcal{X})$ consists of $\alpha$-Lipschitz functions with the norm 
 \begin{equation*}
     \lVert f \rVert_{W_\infty^\alpha}= \lVert f \rVert_\infty + \lvert f \rvert_{W_\infty^{\alpha}}\quad \mbox{with}~\lVert f \rVert_\infty= \sup_{\bm x\in \mathcal{X}} \lvert f(\bm x) \rvert,\quad \lvert f \rvert_{W_\infty^{\alpha}}= \sup_{\bm x\neq \bm y} \frac{\left\vert f(\bm x)-f(\bm y)\right\vert}{\left\| \bm x-\bm y\right\|_2^\alpha}\,.
 \end{equation*}
 Note that $\lvert f \rvert_{W_\infty^{\alpha}}$ is the semi-norm. For $\alpha= s+t$ with $s \in \NN$ and $t \in (0,1]$,  $W_\infty^\alpha (\mathcal{X})$ consists of $s$-times differentiable functions whose partial derivatives of order $s$ are $t$-Lipschitz functions, with an equivalent norm $\lVert f \rVert_{W_\infty^\alpha}= \sum_{\|\bm k\|_2 < s} \lVert D^{\bm k} f \rVert_\infty + \sum_{\|\bm k\|_2 = s} \lVert D^{\bm k} f \rVert_{W_\infty^{t}}$.
 For certain regularity assumptions for the ``target function", we will detail them in the respective sections.

	\section{Main Results for Classification} \label{sectionclass}
	In this section, we demonstrate the main results of adversarial training for classification: there exist infinitely many classifiers obtained by adversarial training with commonly used loss functions, such as hinge loss and logistic loss, that can achieve arbitrarily small adversarial training error and good robust generalization error with the same order of the lower bound when the data distribution is well-separated and of relatively high quality.
	
	\subsection{Problem settings and notations} \label{assumptionclass}
	For binary classification, the objective of learning is to find the best classifier (Bayes rule) $f_c$ on $\mathcal{X}$ defined by
	\begin{equation}
		f_c(\bm x)= \begin{cases}
			1, & \hbox{if} \ \rho(y=1|\bm x)\geq  \rho(y=-1|\bm x)\,, \\
			-1, & \hbox{if} \ \rho(y=1|\bm x)<  \rho(y=-1|\bm x)\,,
		\end{cases}
	\end{equation}
	which is the minimizer of the standard \textit{misclassification error} with the 0-1 loss
	\begin{equation}
		\mathcal{R}(f):=\int_{\mathcal{Z}}  \mathbbm{1}_{\{yf(x)=-1\}} d\rho\,.
	\end{equation}
Based on the observed data sample, a learning algorithm aims at finding a function $f$ in a hypothesis space $\mathcal{F}$ such that the classifier $\text{sgn}(f)$ is a good approximation of the Bayes rule $f_c$. In the next, we introduce the empirical/expected risk for analysis.
	
\noindent {\bf Empirical and expected risks for standard and robust learning:}	Since the 0-1 loss is non-convex and discontinuous, it is difficult to optimize. Instead, one can utilize some \emph{surrogate} loss functions $\phi$ to learn an estimator from the data sample $D$, such as hinge loss, least squares loss, and logistic loss. To be specific, we define the \textit{$\phi$-risk} and its minimum $f_\rho^\phi$ as
	\begin{equation}
		\quad f_\rho^\phi := \argmin_f \mathcal{E}^\phi(f)\,, \quad \mbox{with}~\mathcal{E}^\phi(f):= \int_{\mathcal{Z}}\phi\left(yf(\bm x)\right)  \mathrm{d}\rho\,.
	\end{equation}
 Actually, $f_\rho^\phi$ has the closed form for many commonly used surrogate loss functions \cite{zhang2004statistical}. For example, denoting $\eta(\bm x) :=  \rho(y=1|\bm x)$, we have $f_\rho^\phi = 2\eta -1 = f_\rho$ for the least squares loss; we have $f_\rho^\phi= \text{sgn}(2\eta-1)= f_c$ for the hinge loss. The standard empirical risk minimization (ERM) algorithm aims to minimize the \textit{empirical $\phi$-risk} over the hypothesis space being the deep ReLU FNNs $\mathcal{F}_{\vec{d},L}$
	\begin{equation}
	\widehat{f}_{D,\phi}= \argmin_{f \in \mathcal{F}_{\vec{d},L}} \widehat{\mathcal{E}}_D^\phi(f), \quad \mbox{with}~	\widehat{\mathcal{E}}_D^\phi(f):=\frac{1}{n} \sum_{i=1}^n \phi\left(y_if(\bm x_i)\right)\,.
	\end{equation}
We desire that the classifier  $\text{sgn}(\widehat{f}_{D,\phi})$ can approach the Bayes rule when the number of samples is large enough, in the sense that the standard excess misclassification error $\mathcal{R}(\text{sgn}(\widehat{f}_{D,\phi}))-\mathcal{R}(f_c)$ is small.
	
	The above definitions can be extended to the adversarial training setting where we apply the robust loss instead of the standard loss. In this paper, we consider the $\ell_\infty$ white-box adversarial attack, where the adversary can use small perturbations of the inputs within some $\ell_\infty$ ball to maximize the standard loss. In order to defend against such adversarial attack, our goal is to minimize the \textit{adversarial misclassification error} (robust generalization) 
	\begin{equation}
		\mathcal{R}^\delta (f):=\int_{\mathcal{Z}} \max_{\bm x' \in B_{\delta,\infty}(\bm x)} \mathbbm{1}_{\{yf(\bm x')=-1\}} d\rho\,,
	\end{equation}
	which measures the robust generalization performance, and we denote the best robust classifier as $f_c^\delta= \argmin_f \mathcal{R}^\delta (f)$. Similarly, the adversarial training implements the ERM algorithm by minimizing the \textit{empirical adversarial $\phi$-risk}
	\begin{equation} \label{empiricalclass}
		\widehat{\mathcal{E}_D^{\phi,\delta}} (f):=\frac{1}{n} \sum_{i=1}^n  \max_{\bm x'_i \in B_{\delta,\infty}(\bm x_i)}\phi \left( y_i f(\bm x'_i)\right) \,,
	\end{equation}
	over the hypothesis space $\mathcal{F}_{\vec{d},L}$. Moreover, we denote $\widetilde \Delta_{\delta,\vec d, L}$ as the set of the global minima of the optimization problem \Cref{empiricalclass} for adversarial training, i.e.,
	\begin{equation}
		\widetilde \Delta_{\delta,\vec d, L}:= \left\{ \widehat f_D^\delta: \widehat f_D^\delta = \argmin_{f\in \mathcal{F}_{\vec{d},L}} 	\widehat{\mathcal{E}_D^{\phi,\delta}} (f) \right\}\,.
	\end{equation}

\subsection{Assumptions}
In this subsection, apart from assumptions discussed in \Cref{sec:setting}, we additionally require some assumptions with regard to the data separation and quality. 
All of them are related to how to set the Borel measure $\rho$ to control the data generation process.

First, we make the following assumption related to well-separated data in $\mathcal X$. 

 \begin{assumption}[well separated data] \label{assumption2}
 Denote $A=\{\bm x \in \mathcal{X}: f_c(\bm x)=1\}$ and $B=\{\bm x \in \mathcal{X}: f_c(\bm x)=-1\}$, clearly we have $\mathcal{X}=A \cup B$.
The two classes are $2\delta$-separated if
	\begin{equation} \label{separatedass}
		\|\bm x_A - \bm x_B\|_\infty \geq 2\delta, \quad \forall \bm x_A \in A, \  \bm x_B \in B\,.
	\end{equation}
 \end{assumption}
 \begin{remark}
 This assumption is needed to guarantee the existence of a robust classifier, which is also considered in previous theoretical work \cite{li2022robust}. 
 This assumption has been discussed in the introduction and is demonstrated to be attainable. For real data sets, different classes tend to be well-separated, and the perturbation radius is typically much smaller than the separation distance of different classes \cite{yang2020closer}. For example, on CIFAR-10, the minimum separation distance is 0.21, which is much larger than the perturbation radius $\delta=8/255$. 
 \end{remark}

 However, merely with \Cref{assumption2}, \cite{li2022robust} show that a worst-case requirement on the model complexity suffers from the curse of dimensionality. This is because no regularity assumption is added to the target function. To obtain a relaxed model complexity requirement, apart from the $\alpha$-H\"older continuous assumption as mentioned in \Cref{sec:setting}, we additionally require that the Bayes rule is confident in its prediction.
 \begin{assumption}[regularity assumption and high confidence of the Bayes rule] \label{assumption1}
 We assume that $\eta \in W^\alpha_\infty (\mathcal{X})$ with $\alpha \in \NN$. Besides, there exists some arbitrary small constant $\zeta>0$ such that
    \begin{equation} \label{uncertaintyass}
		\left|\eta(\bm x)-0.5\right| > \zeta, \quad \forall \bm x \in \mathcal{X}\,.
    \end{equation}
 \end{assumption}
 \begin{remark}
    This assumption is also reasonable since it assures that for the true data distribution $\rho$, if the Bayes rule $f_c$ indicates that the label of the input $\bm x$ belongs to one class, then the probability that it belongs to this class should not be that close to $0.5$. That means the true classifier should have some confidence in its classification for every input data. In other words, this assumption ensures that the data distribution is of relatively high quality. Similar notations to measure the quality of data samples and features are utilized in \cite{ilyas2019adversarial,dong2022label}.
 \end{remark}

	\subsection{Generalization analysis of adversarial training global minima on over-parameterized FNNs}
	In this subsection, we indicate that overfitted DNNs in adversarial training can generalize under the circumstances that the data distribution is of relatively high quality and is well-separated, the perturbation radius is small enough, and the number of free parameters in DNNs is large enough. To begin with, we utilize the following error decomposition method for the excess adversarial misclassification error.
	
	\begin{proposition} \label{proposition2}
		Let $f_c^\delta := \argmin_f \mathcal{R}^\delta (f)$. For any classifier $f$, we have
		\begin{equation}
			\mathcal{R}^\delta (f) - \mathcal{R}^\delta (f_c^\delta)  \leq \mathcal{R}^\delta (f) - \mathcal{R} (f) + \mathcal{R} (f) - \mathcal{R} (f_c) \,.
		\end{equation}
        Moreover, we always have $\mathcal{R} (f) \leq \mathcal{R}^\delta (f)$.
	\end{proposition}
	
	\begin{proof} [Proof of \Cref{proposition2}]
		We use the following error decomposition and the fact that $ \mathcal{R} (f_c) \leq \mathcal{R} (f_c^\delta)$ to get
		\begin{equation*}
			\begin{aligned}
				\mathcal{R}^\delta (f) - \mathcal{R}^\delta (f_c^\delta) & = \mathcal{R}^\delta (f) - \mathcal{R} (f) + \mathcal{R} (f) - \mathcal{R} (f_c) \\
				& + \mathcal{R} (f_c)- \mathcal{R} (f_c^\delta) +  \mathcal{R} (f_c^\delta)- \mathcal{R}^\delta (f_c^\delta) \\
				& \leq \mathcal{R}^\delta (f) - \mathcal{R} (f) + \mathcal{R} (f) - \mathcal{R} (f_c) + \mathcal{R} (f_c^\delta)- \mathcal{R}^\delta (f_c^\delta) \,.
			\end{aligned}
		\end{equation*}
		Moreover, for any $ \bm x\in \mathcal{X}, y\in \{-1,1\}$, $\phi(yf(\bm x)) \leq \max_{\bm x' \in B_{\delta,\infty}(\bm x)} \phi(yf(\bm x'))$ holds for any $f$, then we get $\mathcal{R} (f) \leq \mathcal{R}^\delta (f)$. Therefore, we also have $\mathcal{R} (f_c^\delta) \leq \mathcal{R}^\delta (f_c^\delta)$. Thus we complete the proof.
	\end{proof}
    \begin{remark}
    	Since $\mathcal{R} (f) \leq \mathcal{R}^\delta (f)$ for any $f$, the robust generalization of any function is always larger than its standard generalization. This lower bound of robust generalization error together with the upper bound of it in \Cref{proposition2} partially illustrates the robust generalization gap phenomenon. Moreover, \Cref{proposition2} shows that the robust generalization of a classifier $f$ is bounded by the sum of its standard generalization $\mathcal{R} (f) - \mathcal{R} (f_c)$ and its robustness $\mathcal{R}^\delta (f) - \mathcal{R} (f)$. Roughly speaking, the existence of such an additional robustness term of the classifier might result in lower robust generalization performance compared with the standard generalization performance in adversarial training. We explicitly demonstrate the existence of the robust generalization gap hereinafter.
    \end{remark}
	
	Based on the above error decomposition, we are now ready to bound the adversarial misclassification error of good adversarial training classifiers on over-parameterized deep ReLU FNNs. The following theorem is one main result of our paper with the surrogate loss being the hinge loss.
	
	\begin{theorem}[upper bound under the hinge loss] \label{theorem3}
		Let the surrogate loss function $\phi(t)=\max\{1-t,0\}$ be the hinge loss. Suppose that the Borel measure $\rho$ satisfies \Cref{assumption3}, \Cref{assumption2}, \Cref{assumption1} with $\eta \in W^\alpha_\infty (\mathcal{X})$ and $\alpha \in \NN$, taking the perturbation radius $\delta< \frac{q_X}{3} \leq \frac{1}{3} n^{-\frac{1}{d}}$, then for any $C_0 \in (0,1]$, there exist infinity many adversarial training global minima $\widehat{f}_D^{over} \in \tilde \Delta_{\delta, \vec d, L}$ with depth $L= \mathcal{O} \left(  \log \frac{1}{\zeta} \right) $, width $d_1=  \mathcal{O} \left(  \zeta ^{-\frac{d}{\alpha}} \log \frac{1}{\zeta}+ n \right) $, $d_2, \dots, d_L= \mathcal{O} \left( \zeta ^{-\frac{d}{\alpha}} \log \frac{1}{\zeta} \right) $, and non-zero free parameters $ \mathcal{O} \left( \zeta ^{-\frac{d}{\alpha}} \log \frac{1}{\zeta}+ n \right) $, such that
		\begin{equation} \label{standclassupper}
			\mathbb{E} \left[  \mathcal{R}\left( \text{sgn}\left( \widehat{f}_D^{over}\right) \right) -\mathcal{R}(f_c) \right]  \leq 2 \|J_\rho\|  \left( (2+2C_0) \delta\right) ^d n\,,
		\end{equation}
		and
		\begin{equation} \label{adversarialclassupper}
			\mathbb{E} \left[ \mathcal{R}^\delta \left( \text{sgn}\left(\widehat{f}_D^{over}\right) \right)  - \mathcal{R}^\delta (f_c^\delta) \right]  \leq 3 \|J_\rho\|  \left( (4+2C_0) \delta\right) ^d n\,.
		\end{equation}
	\end{theorem}
    \begin{remark}
    We make the following remarks on the derived results:\\
    1) This theorem shows that for the $2\delta$-separated data distribution with relatively high quality (with the quality measured by $\zeta$), when the perturbation radius $\delta$ is small enough, and the complexity of the neural network is large enough depending on the data distribution's quality and regularity, there exist infinitely many adversarial training global minima with clean and robust generalization performance.\\
    2) This result partially indicates the importance of the data distribution's quality in adversarial training, which is consistent with the empirical findings that adversarial training with high-quality data has better robust performance compared with using low-quality data, and it can largely alleviate the robust overfitting problem \cite{dong2021data}.\\
    3) Moreover, when the data distribution's quality is higher ($\zeta$ is larger) or its regularity is larger ($\alpha$ is larger), the requirement of the model complexity is smaller, to ensure the existence of adversarial training global minima with good robust generalization performance. Such requirement of model complexity is better than $\mathcal{O}(\delta^{-d})$ stated in \cite{li2022robust} when the regularity is large. They only consider the worst case of model complexity to obtain good robust generalization for all $2\delta$-separated data, without consideration of the regularity of the target function.
    \end{remark}

    \begin{remark}
    We make the following remarks on the derived bounds w.r.t. the perturbation radius:\\
    1) The perturbation radius plays a significant role in the adversarial training. The well-separated training set is a standard assumption in the over-parameterized literature. Here we require that the perturbation radius is smaller than a third of the training set's separation distance. Such assumption can be satisfied when the input dimension $d$ is large, e.g., for CIFAR-10 dataset $d=3072$, then $n^{-\frac{1}{d}}$ can be nearly 1, while $\delta$ is typically at most 0.05 in practice \cite{guo2017countering}. It is also exhibited in \cite{shafahi2018are} that for the same MNIST task, adversarial training on the MNIST dataset with higher resolution can achieve higher adversarial robustness.\\
    2) Our robust generalization error upper bound suggests that the robust generalization error of adversarial training can be smaller when the perturbation radius is relatively smaller. This is partially derived from the expansion of the memorization of the label noise at one point to the $\delta$ ball while overfitting occurs in adversarial training. Empirical results also suggest that such label noise would be larger with usage of larger perturbation radius in adversarial training, thus resulting in larger variance and the robust overfitting phenomenon \cite{dong2022exploring,dong2022label}.
    \end{remark}

    	Next, we also provide a lower bound of the adversarial misclassification error for all the adversarial training global minima on over-parameterized deep ReLU nets with the hinge loss.
	
	\begin{theorem}[lower bound under the hinge loss] \label{lowerbound2}
		Under the same setting of \Cref{theorem3}, then for any adversarial training global minimum $\widehat{f}_D^{over} \in \tilde \Delta_{\delta, \vec d, L}$ with non-zero free parameters $ \mathcal{O} \left( \zeta ^{-\frac{d}{\alpha}} \log \frac{1}{\zeta}+ n \right) $, we have
		\begin{equation} \label{adversarialclasslower}
			\mathbb{E} \left[ \mathcal{R}^\delta \left( \text{sgn}\left(\widehat{f}_D^{over}\right) \right)  - \mathcal{R}^\delta (f_c^\delta) \right]  \geq 2\zeta \|\bar J_\rho\| \mathcal{R}(f_c) (4 \delta)^d n\,.
		\end{equation}
	\end{theorem}

    \begin{remark}
        Comparing \Cref{standclassupper} with \Cref{adversarialclasslower}, since $C_0 \in (0,1]$ can be arbitrarily small, we have that the excess standard misclassification error of good adversarial training global minima can be smaller than $\mathcal{O}((2\delta)^d n)$, while their excess adversarial misclassification error are larger than $\mathcal{O}((4 \delta)^d n)$. Such difference is because, for the adversarial training global minima, their standard generalization performance would only be influenced by the memorization of noise in the $\delta$-ball around the input training data points, while their robust generalization performance would be influenced by the $2\delta$-ball around the input training data points. This explicitly demonstrates the existence of the robust generalization gap.
    \end{remark}
    
Since $\mathcal{R}(f_c)$ is the misclassification error of the Bayes rule which also measures the quality of the data distribution, this lower bound again demonstrates the importance of the perturbation radius and the data distribution's quality on the adversarial training as is exhibited in \cite{dong2021data,dong2022exploring,dong2022label}. Furthermore, comparing \Cref{adversarialclassupper} with \Cref{adversarialclasslower}, since $C_0 \in (0,1]$ can be arbitrarily small, the excess adversarial misclassification error upper bounds of adversarial training global minima we derived above matches the order of the lower bound, showing that our construction is almost optimal.
	
	Moreover, the above adversarial misclassification error bound results under the hinge loss can be extended to the other commonly used loss functions for the classification tasks, e.g., the logistic loss. We can also demonstrate that there still exist infinitely many adversarial training global minima that can achieve arbitrarily small adversarial training error and good robust generalization error. The details are specified in \Cref{appendixlogistic}.

	\section{Main Results on Regression Tasks}
	\label{section3}
	
	In this section, we extend our results in \Cref{sectionclass} to the regression tasks. Specifically, under the over-parameterized regime, we first demonstrate the existence of infinitely many adversarial training global minima that can achieve near-optimal rates of convergence for the standard generalization, when the perturbation radius is small enough. Then, we also show that there are infinitely many adversarial training global minima that can obtain good robust generalization errors of the same order as the lower bound when the perturbation radius satisfies some conditions.
	\subsection{Notations and assumptions}
	\label{section31}
	For regression under the least squares loss, the objective of learning is to find the target function $f_\rho$, which minimizes the standard \textit{generalization error}
	\begin{equation}
		\mathcal{E}(f):=\int_{\mathcal{Z}} (f(\bm x)-y)^2d\rho\,.
	\end{equation}
 We can write it in the style of the data generation process, i.e., 
 $y = f_{\rho}(\bm x) + \epsilon$ for any data point $(\bm x, y) \sim \rho$, where the noise $\epsilon$ is assumed to have zero mean with $\mathbb{E}[\epsilon]=0$ and bounded variance with $\mathbb{V}[\epsilon]=\sigma^2$.
	
	The ERM algorithm under the least squares loss aims to minimize the \textit{empirical generalization error}
	\begin{equation}
		\widehat{\mathcal{E}}_D (f):= \frac{1}{n} \sum_{i=1}^n (f(\bm x_i)-y_i)^2,
	\end{equation}
	over the hypothesis space $\mathcal{F}_{\vec{d},L}$.
	
	The above definitions can also be extended to the adversarial training setting in the regression task as what is done in \Cref{sectionclass} for the classification task. To defend against the $\ell_\infty$ white-box adversarial attack, our goal is to minimize the \textit{adversarial generalization error} $\mathcal{E}^\delta(f)$ with
	\begin{equation}
	f_\rho^\delta (\bm x) := \argmin_{f}	\mathcal{E}^\delta (f) \quad \mbox{with}~	\mathcal{E}^\delta (f):=\int_{\mathcal{Z}} \max_{\bm x' \in B_{\delta,\infty}(\bm x)} \left(f(\bm x')-y \right)^2 d\rho\,,
	\end{equation}
	where $f_\rho^\delta (\bm x)$ is denoted as the robust target function. Correspondingly, the adversarial training implements the ERM algorithm that minimizes the \textit{empirical adversarial generalization error}
	\begin{equation} \label{regressopt}
		\widehat{\mathcal{E}_D^\delta} (f):=\frac{1}{n} \sum_{i=1}^n \max_{\bm x'_i \in B_{\delta,\infty}(\bm x_i)} \left(f(\bm x'_i)-y_i \right)^2,
	\end{equation}
	over the hypothesis space $\mathcal{F}_{\vec{d},L}$. Moreover, we denote $\Delta_{\delta,\vec d, L}$ as the set of the global minima of the optimization problem \Cref{regressopt} for adversarial training, i.e.,
	\begin{equation}
		\Delta_{\delta,\vec d, L}:= \left\{ \widehat f_D^\delta: \widehat f_D^\delta = \argmin_{f\in \mathcal{F}_{\vec{d},L}} \widehat{\mathcal{E}_D^\delta } (f) \right\}.
	\end{equation}

 For regression, the required assumptions are weaker than that of classification in \Cref{assumptionclass}. The reason is that in the regression task, we measure the squares loss within the $\delta$-ball, which can remain small if $f$ is smooth. Whereas in the classification task, even if $f$ changes a little in the $\delta$-ball, its sign can vary from $-1$ to $+1$ if its value is close to 0.
 Here we only need the distortion assumption in \Cref{assumption3} and the the regularity assumption for $f_{\rho}$ stated in \Cref{sec:setting}.

	\subsection{Standard generalization analysis of adversarial training estimators on over-parameterized FNNs}
	In this subsection, we study the standard generalization error analysis of the estimators obtained by adversarial training on over-parameterized deep ReLU FNNs, and answer the question of whether they can achieve good learning rates as the standard ERM estimators on under-parameterized deep ReLU FNNs do.
	
	Suppose that the target function $f_\rho \in W_\infty^\alpha (\mathcal{X})$ with $\alpha>0$. Denote $\Psi_D$ as the set of regression function estimators that are derived according to the data sample $D$ with size $n$. 
	The classical statistical results \cite{gyorfi2002distribution} demonstrated that the optimal \textit{rates of convergence} that can be achieved by a learning algorithm is
	\begin{equation}
		\inf_{f_D \in \Psi_D} \sup_{f_\rho \in W_\infty^\alpha (\mathcal{X}), \rho_X \in \Phi_\rho} \mathbb{E}\left[ \mathcal{E}(f_D)- \mathcal{E}(f_\rho) \right] =\Theta \left(  n^{-\frac{2\alpha}{2\alpha+d}}\right) .
	\end{equation}
	With a truncated operator introduced for the estimator
	$$\pi_{M} f(\bm x):=\left\{\begin{array}{ll}
		f(\bm x), & \hbox{if} \ |f(\bm x)| \leq M, \\
		M, & \hbox{if} \ f(\bm x)>M, \\
		- M, & \hbox{if} \ f(\bm x) <-M\,,
	\end{array}\right. $$
  recent works indicate that all the truncated estimators (with the truncated operator applied on the estimators) obtained by the standard ERM algorithm on under-parameterized deep ReLU FNNs can achieve near-optimal learning rates \cite{SchmidtHieber2020,han2020depth}.
	
	\begin{lemma} [\cite{SchmidtHieber2020,han2020depth}] \label{lemma1}
		Suppose that the target function $f_\rho \in W_\infty^\alpha (\mathcal{X})$ with $\alpha>0$, there exists some under-parameterized FNN structure $\mathcal{F}_{\vec{d},L}$ with $L \sim \log n$, $d_1 = \mathcal{O}(  n^{\frac{d}{2\alpha+d}})$ , and $d_2, d_3, \dots d_L = \mathcal{O}\left(  \log n\right) $, such that for $f_D^{under}= \pi_M \argmin_{f\in \mathcal{F}_{\vec{d},L}} \widehat{\mathcal{E}}_D (f)$, i.e., any truncated estimators of the standard ERM algorithm, we have
		\begin{equation}
			\sup_{f_\rho \in W_\infty^\alpha (\mathcal{X}), \rho_X \in \Phi_\rho} \mathbb{E}\left[ \mathcal{E}\left(f_D^{under}\right) - \mathcal{E}\left( f_\rho\right)  \right]  \leq C_1 \left(\frac{n}{\log n} \right) ^{-\frac{2\alpha}{2\alpha+d}},
		\end{equation}
		where $C_1$ is a constant independent of $n$.
	\end{lemma}
	
	However, the generalization performance of the global minima of standard ERM algorithms on over-parameterized deep ReLU FNNs is still theoretically unclear. The empirical results exhibit that some ERM global minima on over-parameterized deep ReLU FNNs can not only interpolate the training data but also achieve good generalization performance \cite{belkin2019reconciling,zhang2021understanding},  the occurrence of such benign overfitting phenomenons are further theoretically studied by many other works \cite{bartlett2020benign,cao2021risk,lin2021generalization}.
	
	In this section, we try to further understand the standard generalization performance of adversarial training on over-parameterized FNNs, extending previous work \cite{lin2021generalization} from the standard ERM algorithm to the adversarial training. Our result indicates that under the over-parameterized regime, there do exist infinitely many adversarial training estimators that can achieve zero adversarial training error as well as the near-optimal rates of convergence for the standard generalization error, i.e., clean accuracy is promising, when the perturbation radius is small enough.
	
	\begin{theorem} \label{theorem1}
 Suppose that the target function $f_\rho \in W_\infty^\alpha (\mathcal{X})$ with $\alpha>0$, and the marginal distribution $\rho_X$ satisfies \Cref{assumption3}. If perturbation radius $\delta<  \min \left\{\frac{q_X}{3}, n^{-\frac{2\alpha}{(2\alpha+d)d}-\frac{1}{d}} \right\}$, then there exist infinity many adversarial training estimators $\widehat{f}_D^{over} \in \Delta_{\delta, \vec d, L}$ with depth $L= \mathcal{O} \left( \log n\right) $, and width $d_1= \mathcal{O} \left( n\right)$, $d_2,\dots, d_L= \mathcal{O} \left( \log n\right)$, such that
		\begin{equation}
			\sup_{f_\rho \in W_\infty^\alpha (\mathcal{X}), \rho_X \in \Phi_\rho} \mathbb{E}\left[ \mathcal{E}\left(\widehat f_D^{over}\right) - \mathcal{E}\left( f_\rho\right)  \right]  \leq C_2 \left(\frac{n}{\log n} \right) ^{-\frac{2\alpha}{2\alpha+d}},
		\end{equation}
		where $C_2$ is a constant independent of $n$.
	\end{theorem}
	
	\begin{remark}
Under the over-parameterized regime, \Cref{theorem1} states that there are infinitely many adversarial training estimators of which the construction only depends on the data sample $D$, such that no matter $D$ is drawn from any distribution $\rho$ satisfying the described regularity condition, they can achieve the near-optimal rates of convergence for the standard generalization error. 
However, this is different from \Cref{theorem3}, which describes that for any distribution $\rho$ satisfying the described regularity condition, there exist infinitely many adversarial training global minima, of which the construction depends on both the data sample $D$ and the data distribution $\rho$, such that its standard generalization error bound is \Cref{adversarialclassupper}. This illustrates why the order of the two error bounds is different.
	\end{remark}

 \begin{remark}
     Recent works indicate that adversarial training might result in the robustness-accuracy trade-off, i.e., adversarial training to get a robust network can lead to a drop in the standard test accuracy \cite{zhang2019theoretically,tsipras2019robustness}. However, it is demonstrated in \cite{yang2020closer,li2022robust} that for the classification task, when the data distribution is separable, and the perturbation radius is smaller than the separation distance, the robustness and accuracy are both achievable but the required network size suffers from the curse of dimensionality. Our result confirms this claim for the regression task as well, by indicating that adversarial training is not only a training algorithm that can help to achieve good robustness, but it can also achieve almost optimal standard generalization performance at the same time.
     More importantly, our results demonstrate that linear over-parameterization with $\mathcal{O}(n \log n)$ is sufficient to achieve this statistically. Nevertheless, we only prove the existence of such good adversarial training estimators, and it remains to answer the question of how these adversarial training global minima with good standard generalization performance can be obtained by some optimization algorithms.
 \end{remark}
	
	\subsection{Robust generalization analysis of adversarial training on \\ over-parameterized FNNs}
	In this subsection, we further study the robust generalization performance of the adversarial training global minima. The key idea of the proof is to utilize the following error decomposition method.
	
	\begin{proposition} \label{proposition1}
		Let $f_\rho^\delta := \argmin_{f}	\mathcal{E}^\delta (f)$. For any $f$, we have
		\begin{equation}
			\mathcal{E}^\delta (f) - \mathcal{E}^\delta (f_\rho^\delta) \leq \mathcal{E}^\delta (f) - \mathcal{E} (f) + \mathcal{E} (f) - \mathcal{E} (f_\rho)\,.
		\end{equation}
        Moreover, we always have $\mathcal{E} (f) \leq \mathcal{E}^\delta (f)$.
	\end{proposition}
	
	\begin{proof} [Proof of \Cref{proposition1}]
		We use the following error decomposition and the fact that $ \mathcal{E} (f_\rho) \leq \mathcal{E} (f_\rho^\delta)$ to get
		\begin{equation*}
			\begin{aligned}
				\mathcal{E}^\delta (f) - \mathcal{E}^\delta (f_\rho^\delta) & = \mathcal{E}^\delta (f) - \mathcal{E} (f) + \mathcal{E} (f) - \mathcal{E} (f_\rho) + \mathcal{E} (f_\rho)- \mathcal{E} (f_\rho^\delta) +  \mathcal{E} (f_\rho^\delta)- \mathcal{E}^\delta (f_\rho^\delta) \\
				& \leq \mathcal{E}^\delta (f) - \mathcal{E} (f) + \mathcal{E} (f) - \mathcal{E} (f_\rho) + \mathcal{E} (f_\rho^\delta)- \mathcal{E}^\delta (f_\rho^\delta).
			\end{aligned}
		\end{equation*}
		Moreover, since $\forall \bm x\in \RR^d, y\in \RR$, $\left(f(\bm x)-y \right)^2 \leq \max_{\bm x' \in B_{\delta,\infty}(\bm x)} \left(f(\bm x')-y \right)^2$ holds for any $f$, we get $\mathcal{E} (f) \leq \mathcal{E}^\delta (f)$. Therefore, we further have $\mathcal{E} (f_\rho^\delta) \leq \mathcal{E}^\delta (f_\rho^\delta)$. Thus we complete the proof.
	\end{proof}
	
	Based on the above error decomposition, we are now ready to bound the excess adversarial generalization error of the good adversarial training global minima on over-parameterized deep ReLU FNNs.
	
	\begin{theorem} \label{theorem2}
	Suppose that the target function admits $f_\rho \in W_\infty^\alpha ([0,1]^{d})$ with $\|f\|_{W^\alpha_\infty([0,1]^d} \leq B$ and $\alpha \geq 2$ being an integer, and the marginal distribution of the data satisfies $\rho_X \in \Phi_\rho$ in \Cref{assumption3}. If the radius of adversarial training satisfies $\delta< \frac{q_X}{3} \leq \frac{1}{3} n^{-\frac{1}{d}}$, then $\forall C_0 \in (0,1]$, there exist infinity many adversarial training global minima $\widehat{f}_D^{over} \in \Delta_{\delta, \vec d, L}$, with depth $L= \mathcal{O} \left(  \log \frac{1}{\delta} \right) $, width $d_1=  \mathcal{O} \left( \delta^{-\frac{d}{2\alpha-2}} \log \frac{1}{\delta}+ n \right) $, $d_2, \dots, d_L= \mathcal{O} \left( \delta^{-\frac{d}{2\alpha-2}} \log \frac{1}{\delta} \right) $, and non-zero free parameters $ \mathcal{O} \left( \delta^{-\frac{d}{2\alpha-2}} \log \frac{1}{\delta}+ n \right) $, such that
		\begin{equation} \label{regressionadv}
			\mathbb{E} \left[ \mathcal{E}^\delta (f_{D,\theta,\delta,\tau,\epsilon}^{student}) - \mathcal{E}^\delta (f_\rho^\delta) \right]  \leq C_3 \sqrt{d} \max \left\{ \delta, \left( (4+2C_0) \delta\right) ^d n \right\} .
		\end{equation}
		Moreover, when $n^{-{\frac{1}{d-1}}} \leq \delta< \frac{q_X}{3} \leq \frac{1}{3} n^{-\frac{1}{d}}$, we have
		\begin{equation}
			\mathbb{E} \left[ \mathcal{E}^\delta (f_{D,\theta,\delta,\tau,\epsilon}^{student}) - \mathcal{E}^\delta (f_\rho^\delta) \right]  \leq
			C_3 \sqrt{d}  \left( (4+2C_0) \delta\right) ^d n\,.
		\end{equation}
		where $C_3$ is a constant independent of $d$, $n$ and $\delta$.
	\end{theorem}
	
	This result suggests that for the over-parameterized deep ReLU FNNs, when the perturbation radius is small enough, there exist infinitely many global minima obtained by adversarial training on these FNNs that can achieve good adversarial generalization error. Moreover, the number of parameters to achieve such adversarial generalization error depends on the smoothness of the target function, when $\alpha = \mathcal{O}(d)$ is very large, the required model complexity would be independent of $d$.
	
	The two orders stated in \Cref{regressionadv} come from two parts, one is from the memorization of the noisy labels in the perturbation balls around the input data sample, and another is from the robustness of the target function in the unseen parts of the data, which only depends on the perturbation radius $\delta$ and the smoothness of the target function. Moreover, when the perturbation radius satisfies some constraints, the order of the excess adversarial generalization error upper bound of these good adversarial training global minima in \Cref{theorem2} matches the order of the lower bound, which is stated in the following theorem.
	
	\begin{theorem} \label{lowerbound1}
 Suppose that the target function admits $f_\rho \in W_\infty^\alpha ([0,1]^{d})$ with $\alpha \geq 2$ being an integer, and the marginal distribution of the data satisfies $\rho_X \in \Phi_\rho$ in \Cref{assumption3}.
 If the radius of adversarial training $\delta< \frac{q_X}{3} \leq \frac{1}{3} n^{-\frac{1}{d}}$, then for any adversarial training global minimum $\widehat{f}_D^{over} \in \Delta_{\delta, \vec d, L}$ with non-zero free parameters $ \mathcal{O} \left( \delta^{-\frac{d}{2\alpha-2}} \log \frac{1}{\delta}+ n \right) $, we have
		\begin{equation}
			\begin{aligned}
				\mathbb{E} \left[ \mathcal{E}^\delta (\widehat{f}_D^{over}) - \mathcal{E}^\delta (f_\rho^\delta) \right] &\geq 
				\| \bar J_\rho \|  \sigma^2 (4\delta)^d n - \left[\mathcal{E}^\delta (f_\rho^\delta)- \mathcal{E}\left( f_\rho \right)  \right] \\
				& \geq \| \bar J_\rho \|  \sigma^2 (4\delta)^d n  - \bar C_1 \|J_\rho\| \sqrt{d} \delta
			\end{aligned}
		\end{equation}
		where $	\bar C_1$ is a constant independent of $d$, $n$ and $\delta$.
	\end{theorem}
	
	The minus term in the first line of the lower bound is an unchanged term, which only depends on the intrinsic property of the data distribution. Moreover, since $C_0$ can be arbitrarily small in \Cref{theorem2}, the robust generalization error bound of the good adversarial training global minima shown in \Cref{regressionadv} matches this lower bound when $n^{-{\frac{1}{d-1}}} \leq \delta< \frac{q_X}{3} \leq \frac{1}{3} n^{-\frac{1}{d}}$. Such lower bound also demonstrates the impact of the 
	data quality on the adversarial training, since the variance of the noise $\sigma^2$ is included in the bound.

    \section{Conclusion} \label{conclusion}
    In this paper, we try to answer the question of whether overfitted DNNs in adversarial training can generalize in the over-parameterized setting, i.e., whether there exist adversarial training global minima that can achieve both arbitrarily small training error as well as good robust generalization performance. We study this question for both the classification tasks and the regression tasks.

    For the classification tasks, when the data distribution is well-separated, of relatively high quality, the perturbation radius is small enough, and the model complexity is large enough, we prove the existence of infinitely many adversarial training global minima that can achieve arbitrarily small training error as well as good robust generalization performance. The requirement of the model complexity can be relaxed when the regularity of the target function is larger or the data quality is higher. Our construction of such adversarial training global minima is almost optimal since its robust generalization error bound matches the order of the lower bound. We also demonstrate the existence of the robust generalization gap since the robust generalization has a larger order than the standard generalization even for these almost optimally constructed adversarial training global minima.

    For the regression tasks, we first study the question of whether the robustness-accuracy trade-off can be avoided, i.e., whether adversarial training harms the standard generalization performance. Our results indicate that there are infinitely many adversarial training estimators that can achieve zero adversarial training error as well as near-optimal rates of convergence for the standard generalization error if the perturbation radius is small enough, with only linear over-parameterization. Furthermore, we also study the robust performance, where we also demonstrate infinitely many adversarial training global minima with good robust generalization, which matches the lower bound when the perturbation radius is not that small.
    
    \section*{Acknowledgments}
    
    The research leading to these results received funding from the European Research Council under the European Union’s Horizon 2020 research and innovation program/ERC Advanced Grant E-DUALITY (787960). This article reflects only the authors’ views, and the EU is not liable for any use that may be made of the contained information; Flemish government (AI Research Program); Leuven.AI Institute. Fanghui is supported by UK-Italy Trustworthy AI Visiting Researcher Programme.
	
	\appendix 
	
	\section* {Appendix}
	
	\section{Proof of Main Results in \Cref{sectionclass}}
	
	\subsection{Proof of \Cref{theorem3}} \label{appendixA1}
	The proof of \Cref{theorem3} follows the teacher-student network scheme to construct the adversarial training global minima that both interpolates the training samples within the adversarial perturbations and achieves good standard generalization performance. The main proof techniques are the localized approximation \cite{chui1994neural,chui2020realization} and the product-gate property of deep ReLU FNNs \cite{Yarotsky2017}.
	
	We first introduce the localized approximation approach. Let $\theta, a,b \in \mathbb{R}$ with $a<b$, denote the trapezoid-shaped function $T_{\theta,a,b}$ on $\RR$ with a parameter $0<\theta \leq 1$ as
	\begin{equation}
		T_{a,b,\theta}(t):= \frac{1}{\theta} \left\{ \sigma(t-a+\theta)-\sigma(t-a)-\sigma(t-b)+\sigma(t-b-\theta) \right\}, \quad t\in \RR\,.
	\end{equation}
	For $\bm x=\left(x_1,\dots,x_d\right)\in \RR^d$, denote
	\begin{equation}
		\Gamma_{a,b,\theta}(\bm x):=\sigma\left( \sum_{k=1}^d T_{a,b,\theta}\left(x_k\right)-(d-1) \right)\,,
	\end{equation}
	it is in fact a two-layer ReLU net with the hidden width $4d$.
	It can be easily shown that $0\leq \Gamma_{a,b,\theta}(\bm x)\leq 1$ for all $\bm x\in \mathbb{I}^d$ and
	\begin{equation} \label{Gammafun}
		\Gamma_{a,b,\theta}(\bm x)= \left\{\begin{aligned}
			0, \quad &\hbox{if}\ \bm x \notin[a-\theta,b+\theta]^d, \\
			1, \quad &\hbox{if} \ \bm x \in [a,b]^d.
		\end{aligned}\right.
	\end{equation}
	
	Moreover, in the following of the paper, for any $\bm x \in \RR^d$ and $a \in \RR$, we denote 
	\begin{equation}
		[\bm x- a, \bm x +a]^d:= \bm x + [-a,a]^d.
	\end{equation}
	
	The following lemma indicates the product-gate property of the deep ReLU FNNs which can be found in \cite{Yarotsky2017}.
	\begin{lemma} \label{lemma2}
		For any $\epsilon \in (0,1)$, there exists a deep ReLU FNN $\tilde \times_\epsilon : \RR^2 \to \RR$ with depth and free parameters $\mathcal{O}\left( \log \frac{1}{\epsilon} \right)$ such that 
		\begin{equation}
			\left| \tilde \times_\epsilon (x_1,x_2)- x_1 x_2 \right| \leq \epsilon, \quad \forall x_1,x_2 \in [-1,1].
		\end{equation}
		Moreover, $ \tilde \times_\epsilon (x_1,x_2)=0$ if $x_1=0$ or $x_2=0$.
	\end{lemma}
	
	The following lemma from \cite[Theorem 1]{Yarotsky2017} describes approximation rates of deep ReLU FNNs for Sobolev functions with respect to $L_\infty$ norms.
	\begin{lemma} \label{lemma4}\cite[Theorem 1]{Yarotsky2017}
		Let $\alpha \in \NN$. Suppose that $f \in W^\alpha_\infty([0,1]^d)$ with $\|f\|_{W^\alpha_\infty([0,1]^d)} \leq 1$. Then there exists a deep ReLU FNN $\widehat f$ with depth $L=  \mathcal{O}\left( \log \frac{1}{\epsilon}\right) $, width $d_1, \dots, d_L=   \mathcal{O}\left( \epsilon^{-\frac{d}{\alpha}} \log \frac{1}{\epsilon}\right) $,  and non-zero free parameters $ \mathcal{O}\left( \epsilon^{-\frac{d}{\alpha}} \log \frac{1}{\epsilon}\right)$, such that
		\begin{equation}
			\left\| \widehat f - f \right\|_{L_\infty([0,1]^d)} \leq \epsilon\,.
		\end{equation}
	\end{lemma}
	
	We also need the following comparison theorem for the hinge loss studied in \cite{bartlett2006convexity,chen2004support}, which describes the relationship between the excess misclassification error and the excess $\phi$-risk.
	
	\begin{lemma} \label{lemma5}
		If $\phi$ is the hinge loss $\phi(t)=\max\{1-t,0\}$, then for any measurable function $f:X\to\RR$, there holds
		\begin{equation}
			\mathcal{R}\left( \text{sgn}(f)\right) -\mathcal{R}(f_c)\leq \mathcal{E}^{\phi}\left( \text{sgn}(f)\right)  - \mathcal{E}^{\phi}(f_c)\,.
		\end{equation}
	\end{lemma}
	
	We are now ready to prove \Cref{theorem3} based on \Cref{proposition2}, \Cref{lemma2}, \Cref{lemma4}, and \Cref{lemma5}.
	
	\begin{proof} [Proof of \Cref{theorem3}]
		Note that the target function $f_\rho=2\eta -1 \in W^\alpha_\infty(\mathcal{X})$, and $f_c= \text{sgn}(f_\rho)$. Moreover, by~\Cref{lemma4}, there exists a deep ReLU FNN $\hat f_\theta$ with depth $L=   \mathcal{O}\left( \log \frac{1}{\theta}\right) $, width $d_1, \dots, d_L=   \mathcal{O}\left( \theta^{-\frac{d}{\alpha}} \log \frac{1}{\theta}\right) $,  and non-zero free parameters $  \mathcal{O}\left( \theta^{-\frac{d}{\alpha}} \log \frac{1}{\theta}\right)$, such that
		\begin{equation} \label{errorA32}
			\left\| \hat f_\theta- f_\rho \right\|_{L_\infty(\mathcal{X})} \leq \theta\,.
		\end{equation}
		Denote $c_5= \left\| \hat f_\theta \right\|_{L^\infty(\mathcal{X})} \leq \left\| \hat f_\theta- f_\rho \right\|_{L_\infty(\mathcal{X})} + \left\| \hat f_\rho \right\|_{L_\infty(\mathcal{X})}$, we have $c_5 \leq 2$ since $|f_\rho| \leq 1$. We use $\hat f_\theta$ as the teacher network, and construct the student network $f_{D,\theta,\delta,\tau,\epsilon}^{\phi}$ which is the adversarial training global minimum
		\begin{equation} \label{fDLS}
			f_{D,\theta,\delta,\tau,\epsilon}^{\phi}(\bm x):= \sum_{i=1}^n y_i  \Gamma_{\bm x_i-\delta,\bm x_i+\delta,\tau}(\bm x ) + c_5 \tilde \times_\epsilon \left(\frac{\hat f_\theta(\bm x)}{c_5}, 1- \sum_{i=1}^n  \Gamma_{\bm x_i-\delta,\bm x_i+\delta,\tau}(\bm x ) \right).
		\end{equation}
		When $\bm x\in [\bm x_i-\delta,\bm x_i+\delta]^d$, i.e., $\|\bm x-\bm x_i\|_\infty \leq \delta$, we have
		$\Gamma_{\bm x_i-\delta,\bm x_i+\delta,\tau}(\bm x)= 1$. Moreover, choosing $\tau \leq C_0 \delta< \frac{C_0 q_X}{3}$, we further have $\Gamma_{\bm x_j-\delta,\bm x_j+\delta,\tau}(\bm x)= 0$ for all $j \neq i$. Thereby, $1- \sum_{i=1}^n  \Gamma_{\bm x_i-\delta,\bm x_i+\delta,\tau}(\bm x ) = 0$, we get
		$\tilde \times_\epsilon \left(\frac{\hat f_\theta(\bm x)}{c_5}, 1- \sum_{i=1}^n  \Gamma_{\bm x_i-\delta,\bm x_i+\delta,\tau}(\bm x ) \right) =0$ by \Cref{lemma2}. Therefore,
		\begin{equation} \label{fDLSproperty}
			f_{D,\theta,\delta,\tau,\epsilon}^{\phi} (\bm x) = y_i, \quad \hbox{when} \ \bm x\in [\bm x_i-\delta,\bm x_i+\delta]^d\,.
		\end{equation}
		This suggests that $\widehat{\mathcal{E}_D^{\phi,\delta}} \left( f_{D,\theta,\delta,\tau,\epsilon}^{\phi}\right) =\frac{1}{n} \sum_{i=1}^n  \max_{\bm x'_i \in B_{\delta,\infty}(\bm x_i)}\phi \left( y_i f_{D,\theta,\delta,\tau,\epsilon}^{\phi}(\bm x'_i)\right) =0$, thus $f_{D,\theta,\delta,\tau,\epsilon}^{\phi}$ is indeed the global minimum of adversarial training with the surrogate loss $\phi$. Moreover, $f_{D,\theta,\delta,\tau,\epsilon}^{\phi}$ is a deep ReLU FNN with depth $L= \mathcal{O} \left(  \log \frac{1}{\theta} + \log \frac{1}{\epsilon}\right) $, width $d_1=  \mathcal{O} \left( \theta^{-\frac{d}{\alpha}} \log \frac{1}{\theta}+ n + \log \frac{1}{\epsilon}\right) $, $d_2, \dots, d_L= \mathcal{O} \left( \theta^{-\frac{d}{\alpha}} \log \frac{1}{\theta} + \log \frac{1}{\epsilon}\right) $, and non-zero free parameters $ \mathcal{O}\left( \theta^{-\frac{d}{\alpha}} \log \frac{1}{\theta}+  n + \log \frac{1}{\epsilon} \right)$.
		
		We then bound the excess adversarial misclassification error of this adversarial training classifier.
		By \Cref{proposition2}, we only need to bound two error terms: $\mathcal{R}^\delta \left( \text{sgn}\left(f_{D,\theta,\delta,\tau,\epsilon}^{\phi} \right) \right)  - \mathcal{R} \left( \text{sgn}\left(f_{D,\theta,\delta,\tau,\epsilon}^{\phi} \right) \right)$ and  $\mathcal{R} \left( \text{sgn}\left(f_{D,\theta,\delta,\tau,\epsilon}^{\phi} \right) \right)- \mathcal{R} (f_c)$. 
		
		We first consider the second error term.
		By \Cref{lemma5}, we have 
		\begin{equation*}
			\mathcal{R}\left( \text{sgn}\left( f_{D,\theta,\delta,\tau,\epsilon}^{\phi}\right) \right) -\mathcal{R}(f_c)\leq \mathcal{E}^{\phi}\left( \text{sgn} \left( f_{D,\theta,\delta,\tau,\epsilon}^{\phi}\right)\right)  - \mathcal{E}^{\phi}\left( f_c\right)\,.
		\end{equation*}
		To bound $\mathcal{E}^{\phi}\left( \text{sgn} \left( f_{D,\theta,\delta,\tau,\epsilon}^{\phi}\right)\right)  - \mathcal{E}^{\phi}\left( f_c\right)$, because of $\text{sgn} \left( f_{D,\theta,\delta,\tau,\epsilon}^{\phi}\right) \in [-1,1]$, we have
		\begin{equation*}
			\phi\left( y \cdot \text{sgn} \left( f_{D,\theta,\delta,\tau,\epsilon}^{\phi}\right)(\bm x)\right) - \phi\left(y f_c(\bm x) \right) = y\left(f_c(\bm x)- \text{sgn} \left( f_{D,\theta,\delta,\tau,\epsilon}^{\phi}\right)(\bm x) \right).
		\end{equation*}
		It follows that
		\begin{equation*}
			\mathcal{E}^{\phi}\left( \text{sgn} \left( f_{D,\theta,\delta,\tau,\epsilon}^{\phi}\right)\right)  - \mathcal{E}^{\phi}\left( f_c\right) = \int_{\mathcal{X}} \left(f_c(\bm x)- \text{sgn} \left( f_{D,\theta,\delta,\tau,\epsilon}^{\phi}\right)(\bm x) \right) f_\rho(\bm x) \mathrm{d}\rho_X.
		\end{equation*}
		Denote
		\begin{equation*} 
			f_{D,\theta,\delta,\tau}^{\phi}(\bm x):= \sum_{i=1}^n y_i  \Gamma_{\bm x_i-\delta,\bm x_i+\delta,\tau}(\bm x ) + \hat f_\theta(\bm x) \left(  1- \sum_{i=1}^n  \Gamma_{\bm x_i-\delta,\bm x_i+\delta,\tau}(\bm x ) \right).
		\end{equation*}
		Since $\delta< \frac{q_X}{3}$, we have $1- \sum_{i=1}^n  \Gamma_{\bm x_i-\delta,\bm x_i+\delta,\tau}(\bm x ) \in [0,1]$, then by \Cref{lemma2}, we have
		\begin{equation*} 
			\left\|  f_{D,\theta,\delta,\tau,\epsilon}^{\phi}- f_{D,\theta,\delta,\tau}^{\phi}  \right\|_{L^\infty(\mathcal{X})} \leq c_5 \epsilon\,.
		\end{equation*}
		Notice that when $\bm x\in \mathcal{X} \backslash \left( \cup_{i\in \{1,\dots,n\}} [\bm x_i-\delta-\tau,\bm x_i+\delta+\tau]^d \right) $,  $\Gamma_{\bm x_i-\delta,\bm x_i+\delta,\tau}(\bm x)=0$ for all $i$, we have $f_{D,\theta,\delta,\tau}^{\phi}(\bm x)=  \hat f_\theta(\bm x)$. It follows from \eqref{errorA32} that
		\begin{equation} \label{samesign}
			\left|f_{D,\theta,\delta,\tau,\epsilon}^{\phi} (\bm x)- f_\rho(\bm x)   \right| \leq	\left|f_{D,\theta,\delta,\tau,\epsilon}^{\phi} (\bm x)- f_{D,\theta,\delta,\tau}^{\phi}(\bm x)   \right| + \left|\hat f_\theta (\bm x)- f_\rho(\bm x)   \right| \leq c_5 \epsilon + \theta \,.
		\end{equation}
		By choosing $\epsilon= \theta= \frac{2}{3} \zeta$, we have $\left|f_{D,\theta,\delta,\tau,\epsilon}^{\phi} (\bm x)- f_\rho(\bm x)   \right| \leq 2 \zeta$. Moreover, by \eqref{uncertaintyass} in \Cref{assumption1}, we have $|f_\rho|= |2\eta -1| > 2\zeta$. Therefore, $f_{D,\theta,\delta,\tau,\epsilon}^{\phi} (\bm x)$ has the same sign with $f_\rho(\bm x)$, i.e., $ \text{sgn} \left( f_{D,\theta,\delta,\tau,\epsilon}^{\phi}\right) (\bm x)= f_c(\bm x)$, it follows that
		\begin{equation*}
			\int_{\mathcal{X} \backslash \left( \cup_{i\in \{1,\dots,n\}} [\bm x_i-\delta-\tau,\bm x_i+\delta+\tau]^d \right)} \left(f_c(\bm x)- \text{sgn} \left( f_{D,\theta,\delta,\tau,\epsilon}^{\phi}\right)(\bm x) \right) f_\rho(\bm x) \mathrm{d}\rho_X= 0\,.
		\end{equation*}
		Furthermore, due to $\tau \leq C_0 \delta$ and $|f_\rho|\leq 1$, we have
		\begin{equation*}
			\begin{aligned}
				& \sum_{i=1}^n \int_{ [\bm x_i-\delta-\tau,\bm x_i+\delta+\tau]^d} \left(f_c(\bm x)- \text{sgn} \left( f_{D,\theta,\delta,\tau,\epsilon}^{\phi}\right)(\bm x) \right) f_\rho(\bm x) \mathrm{d}\rho_X \\
				\leq & \sum_{i=1}^n \|J_\rho\| \int_{ [\bm x_i-\delta-\tau,\bm x_i+\delta+\tau]^d} \left(f_c(\bm x)- \text{sgn} \left( f_{D,\theta,\delta,\tau,\epsilon}^{\phi}\right)(\bm x) \right) f_\rho(\bm x) \mathrm{d} \bm x \\
				\leq & 2 \|J_\rho\|  \left( (2+2C_0) \delta\right) ^d n\,.
			\end{aligned}
		\end{equation*}
		Combining these two terms, we get
		\begin{equation*}
			\mathcal{E}^{\phi}\left( \text{sgn} \left( f_{D,\theta,\delta,\tau,\epsilon}^{\phi}\right)\right)  - \mathcal{E}^{\phi}\left( f_c\right) \leq 2 \|J_\rho\|  \left( (2+2C_0) \delta\right) ^d n\,.
		\end{equation*}
		Therefore, we finally have
		\begin{equation} \label{standardupperbound}
			\mathcal{R}\left( \text{sgn}\left( f_{D,\theta,\delta,\tau,\epsilon}^{\phi}\right) \right) -\mathcal{R}(f_c)\leq 2 \|J_\rho\|  \left( (2+2C_0) \delta\right) ^d n\,.
		\end{equation}
		
		Next, we consider the first error term. Notice that
		\begin{equation*}
			\begin{aligned}
				& \mathcal{R}^\delta  \left(\text{sgn}\left( f_{D,\theta,\delta,\tau,\epsilon}^{\phi}\right) \right)  - \mathcal{R} \left( \text{sgn}\left( f_{D,\theta,\delta,\tau,\epsilon}^{\phi}\right) \right) \\
				= & \int_{\mathcal{Z}} \max_{\bm x' \in B_{\delta,\infty}(\bm x)} \mathbbm{1}_{\left\{y \cdot \text{sgn}\left( f_{D,\theta,\delta,\tau,\epsilon}^{\phi}\right) (\bm x') =-1\right\}} - \mathbbm{1}_{\left\{y \cdot \text{sgn}\left( f_{D,\theta,\delta,\tau,\epsilon}^{\phi}\right) (\bm x) =-1\right\}} \mathrm{d}\rho \\
				\leq & \|J_\rho\| \int_{\mathcal{X}} \int_{Y} \max_{\bm x' \in B_{\delta,\infty}(\bm x)} \mathbbm{1}_{\left\{y \cdot \text{sgn}\left( f_{D,\theta,\delta,\tau,\epsilon}^{\phi}\right) (\bm x') =-1\right\}} \\
				& - \mathbbm{1}_{\left\{y \cdot \text{sgn}\left( f_{D,\theta,\delta,\tau,\epsilon}^{\phi}\right) (\bm x) =-1\right\}} \mathrm{d}\rho(y|\bm x) \mathrm{d} \bm x\,.
			\end{aligned}
		\end{equation*}
		To bound this term, we divide $\mathcal{X}$ to two disjoint parts: $\cup_{i\in \{1,\dots,n\}} [\bm x_i-2\delta-\tau,\bm x_i+2\delta+\tau]^d$ and $\mathcal{X} \backslash \cup_{i\in \{1,\dots,n\}} [\bm x_i-2\delta-\tau,\bm x_i+2\delta+\tau]^d$. We first consider the second part, as shown before, for any $\bm x \in \mathcal{X} \backslash \cup_{i\in \{1,\dots,n\}} [\bm x_i-\delta-\tau,\bm x_i+\delta+\tau]^d$ $f_{D,\theta,\delta,\tau,\epsilon}^{\phi} (\bm x)$ has the same sign with $f_\rho(\bm x)$, i.e., $ \text{sgn} \left( f_{D,\theta,\delta,\tau,\epsilon}^{\phi}\right) (\bm x)= f_c(\bm x)$. Therefore, for  any $\bm x \in \mathcal{X} \backslash \cup_{i\in \{1,\dots,n\}} [\bm x_i-2\delta-\tau,\bm x_i+2\delta+\tau]^d$, and any $\bm x' \in B_{\delta,\infty}(\bm x)$, we get $\text{sgn}\left(f_{D,\theta,\delta,\tau,\epsilon}^{\phi} \right) (\bm x')=   f_c (\bm x')$. Furthermore, according to the separated data assumption \eqref{separatedass} in \Cref{assumption2}, $f_c$ will not change the sign in each $L_\infty$ ball with radius $\delta$, i.e., for any $\bm x' \in B_{\delta,\infty}(\bm x)$, $f_c(\bm x')= f_c(\bm x)$. Thus for any $\bm x' \in B_{\delta,\infty}(\bm x)$, $\text{sgn}\left(f_{D,\theta,\delta,\tau,\epsilon}^{\phi} \right) (\bm x')= \text{sgn}\left(f_{D,\theta,\delta,\tau,\epsilon}^{\phi} \right) (\bm x)$. This indicates that
		\begin{equation*}
			\begin{aligned}
				& \int_{\mathcal{X} \backslash \cup_{i\in \{1,\dots,n\}} [\bm x_i-2\delta-\tau,\bm x_i+2\delta+\tau]^d}  \int_{Y} \max_{\bm x' \in B_{\delta,\infty}(\bm x)} \mathbbm{1}_{\left\{y \cdot \text{sgn}\left( f_{D,\theta,\delta,\tau,\epsilon}^{\phi}\right) (\bm x') =-1\right\}} - \\
				& \mathbbm{1}_{\left\{y \cdot \text{sgn}\left( f_{D,\theta,\delta,\tau,\epsilon}^{\phi}\right) (\bm x) =-1\right\}}  \mathrm{d}\rho(y|\bm x)  \mathrm{d}\bm x =0.
			\end{aligned}
		\end{equation*}
		As for the first part, we have
		\begin{equation*}
			\begin{aligned}
				& \sum_{i=1}^n \int_{ [\bm x_i-2\delta-\tau,\bm x_i+2\delta+\tau]^d}    \int_{Y} \max_{\bm x' \in B_{\delta,\infty}(\bm x)} \mathbbm{1}_{\left\{y \cdot \text{sgn}\left( f_{D,\theta,\delta,\tau,\epsilon}^{\phi}\right) (\bm x') =-1\right\}} - \\
				& \mathbbm{1}_{\left\{y \cdot \text{sgn}\left( f_{D,\theta,\delta,\tau,\epsilon}^{\phi}\right) (\bm x) =-1\right\}}  \mathrm{d}\rho(y|\bm x)  \mathrm{d}\bm x \leq  \left( (4+2C_0) \delta\right) ^d n\,.
			\end{aligned}
		\end{equation*}
		Combining these two terms, we get
		\begin{equation}
			\mathcal{R}^\delta  \left(\text{sgn}\left( f_{D,\theta,\delta,\tau,\epsilon}^{\phi}\right) \right)  - \mathcal{R} \left( \text{sgn}\left( f_{D,\theta,\delta,\tau,\epsilon}^{\phi}\right) \right)  \leq \|J_\rho\|    \left( (4+2C_0) \delta\right) ^d n\,.
		\end{equation}
		
		Finally, by \Cref{proposition2}, we get
		\begin{equation*}
			\begin{aligned}
				& \mathcal{R}^\delta \left( \text{sgn}\left(f_{D,\theta,\delta,\tau,\epsilon}^{\phi}\right) \right)  - \mathcal{R}^\delta (f_c^\delta) \\ 
				\leq & \mathcal{R}^\delta \left( \text{sgn}\left(f_{D,\theta,\delta,\tau,\epsilon}^{\phi} \right) \right)  - \mathcal{R} \left( \text{sgn}\left(f_{D,\theta,\delta,\tau,\epsilon}^{\phi} \right) \right) + \mathcal{R} \left( \text{sgn}\left(f_{D,\theta,\delta,\tau,\epsilon}^{\phi} \right) \right)- \mathcal{R} (f_c) \\
				\leq & 2 \|J_\rho\|  \left( (2+2C_0) \delta\right) ^d n + \|J_\rho\|   \left( (4+2C_0) \delta\right) ^d n\,.
			\end{aligned}
		\end{equation*}
		Furthermore, since we choose $\epsilon= \theta= \frac{2}{3} \zeta$, $f_{D,\theta,\delta,\tau,\epsilon}^{student}$ is a deep ReLU FNN with depth $L= \mathcal{O} \left(  \log \frac{1}{\zeta} \right) $, width $d_1=  \mathcal{O} \left(  \zeta ^{-\frac{d}{\alpha}} \log \frac{1}{\zeta}+ n \right) $, $d_2, \dots, d_L= \mathcal{O} \left( \zeta ^{-\frac{d}{\alpha}} \log \frac{1}{\zeta} \right) $, and non-zero free parameters $ \mathcal{O} \left( \zeta ^{-\frac{d}{\alpha}} \log \frac{1}{\zeta}+ n \right) $. Furthermore, we have the adversarial misclassification error bound
		\begin{equation}
			\mathbb{E} \left[ \mathcal{R}^\delta \left( \text{sgn}\left(f_{D,\theta,\delta,\tau,\epsilon}^{\phi}\right) \right)  - \mathcal{R}^\delta (f_c^\delta) \right]  \leq 3 \|J_\rho\|  \left( (4+2C_0) \delta\right) ^d n\,.
		\end{equation}
		Moreover, since $\tau \leq C_0 \delta$ can be arbitrarily chosen, we conclude that there are infinitely many global minima $f_{D,\theta,\delta,\tau,\epsilon}^{\phi} \in \tilde \Delta_{\delta,\vec d, L}$ that can achieve such adversarial misclassification error bound. 
		Thus we complete the proof.
	\end{proof}

 	\subsection{Proof of \Cref{lowerbound2}}
	\begin{proof} [Proof of \Cref{lowerbound2}]
		Notice that according to the separated data assumption \eqref{separatedass} in \Cref{assumption2}, we in fact have $ \mathcal{R}^\delta (f_c) =  \mathcal{R} (f_c)$, this further indicates that $f_c^\delta= f_c$. Moreover, by \eqref{uncertaintyass} in \Cref{assumption1} that $|\eta(\bm x)-0.5|>\zeta$, we have $\max \{\eta(\bm x),1-\eta(\bm x) \} \geq \min \{\eta(\bm x),1-\eta(\bm x) \} + 2\zeta$. Therefore, for any adversarial training global minimum $ \widehat f_{D}^{over} \in \tilde \Delta_{\delta,\vec d, L}$ with non-zero free parameters $ \mathcal{O} \left( \zeta ^{-\frac{d}{\alpha}} \log \frac{1}{\zeta}+ n \right) $, since it can achieve zero adversarial training error as is constructed in \Cref{appendixA1}, we have $\widehat f_{D}^{over}  (\bm x) = y_i$, when $\bm x \in B_{\delta,\infty}(\bm x_i)$. Therefore,
		\begin{equation*}
			\begin{aligned}
				&   \mathbb E \left[ \int_{\bm x \in [\bm x_i-2\delta,\bm x_i+2\delta]^d} \int_Y \max_{\bm x' \in B_{\delta,\infty}(\bm x)} \mathbbm{1}_{\left\{y \cdot \text{sgn}\left( \widehat f_{D}^{over}\right)  (\bm x')  =-1\right\}} \mathrm{d}\rho \right]  \\
				\geq &  \mathbb E \left[ \int_{\bm x \in [\bm x_i-2\delta,\bm x_i+2\delta]^d} \int_Y \mathbbm{1}_{\left\{y y_i =-1\right\}} \mathrm{d}\rho \right]  \\
				= & \int_{\bm x_i}  \int_{\bm x \in [\bm x_i-2\delta,\bm x_i+2\delta]^d} \eta(\bm x)(1-\eta(\bm x_i)) + (1-\eta(\bm x))\eta(\bm x_i) \mathrm{d}\rho_X \mathrm{d}\rho_X \\
				= & \int_{\bm x_i}  \int_{\bm x \in [\bm x_i-2\delta,\bm x_i+2\delta]^d} \min \{\eta(\bm x),1-\eta(\bm x) \} (1-\eta(\bm x_i))  + \min \{\eta(\bm x),1-\eta(\bm x) \} \eta(\bm x_i) \\
				& + \left(  \max \{\eta(\bm x),1-\eta(\bm x) \}-  \min \{\eta(\bm x),1-\eta(\bm x) \}\right)   \min \{\eta(\bm x_i),1-\eta(\bm x_i) \} \mathrm{d}\rho_X \mathrm{d}\rho_X \\
				\geq & \mathbb{E}\left[  \int_{\bm x \in [\bm x_i-2\delta,\bm x_i+2\delta]^d} \min \{\eta(\bm x),1-\eta(\bm x) \} \mathrm{d}\rho_X \right]  \\
				+ & 2\zeta \int_{\bm x_i}  \int_{\bm x \in [\bm x_i-2\delta,\bm x_i+2\delta]^d} \min \{\eta(\bm x_i),1-\eta(\bm x_i) \} \mathrm{d}\rho_X \mathrm{d}\rho_X \\
				= &  \mathbb{E}\left[ \int_{\bm x \in [\bm x_i-2\delta,\bm x_i+2\delta]^d} \int_Y \mathbbm{1}_{\left\{y \cdot f_c(\bm x)=-1\right\}} \mathrm{d}\rho \right]  \\
				+ & 2\zeta \mathbb{E} \left[  \int_{\bm x \in [\bm x_i-2\delta,\bm x_i+2\delta]^d} \min \{\eta(\bm x_i),1-\eta(\bm x_i) \} \mathrm{d}\rho_X \right] \,,
			\end{aligned}
		\end{equation*}
		where the second equality is because $\eta(\bm x)-0.5$ has the same sign with $\eta(\bm x_i)-0.5$ due to the separated data assumption \eqref{separatedass} in \Cref{assumption2}, thus the larger one in $\{\eta(\bm x),1-\eta(\bm x) \}$ multiplies with the smaller one in $\{\eta(\bm x_i),1-\eta(\bm x_i) \}$. It follows that
		\begin{equation*}
			\begin{aligned}
				& \mathbb E \left[\mathcal{R}^\delta \left(  \text{sgn}\left(  \widehat f_{D}^{over}\right) \right)  -\mathcal{R}^\delta (f_c^\delta) \right]  = \mathbb E \left[\mathcal{R}^\delta \left(  \text{sgn}\left(  \widehat f_{D}^{over}\right) \right) -\mathcal{R} (f_c) \right]  \\
				\geq & \sum_{i=1}^n \mathbb E \left[ \int_{\bm x \in [\bm x_i-2\delta,\bm x_i+2\delta]^d} \int_Y \max_{\bm x' \in B_{\delta,\infty}(\bm x)} \mathbbm{1}_{\left\{y \cdot \text{sgn}\left( \widehat f_{D}^{over}\right) (\bm x')  =-1\right\}} - \mathbbm{1}_{\left\{y \cdot f_c(\bm x)=-1\right\}} \mathrm{d}\rho \right] \\
				\geq & \sum_{i=1}^n 2\zeta \mathbb{E} \left[ \int_{\bm x \in [\bm x_i-2\delta,\bm x_i+2\delta]^d} \min \{\eta(\bm x_i),1-\eta(\bm x_i) \} \mathrm{d}\rho_X \right] \\
				\geq & \sum_{i=1}^n 2\zeta \|\bar J_\rho\| (4 \delta)^d \mathbb{E} \left[\min \{\eta(\bm x_i),1-\eta(\bm x_i) \}  \right] \\
				= & 2\zeta \|\bar J_\rho\| \mathcal{R}(f_c) (4 \delta)^d n.
			\end{aligned}
		\end{equation*}
		Thus we complete the proof.
	\end{proof}
	
	\subsection{Robust generalization for adversarial training with logistic loss} \label{appendixlogistic}

 	\begin{theorem}[upper bound under the logistic loss] \label{theorem4}
		Let  $\alpha \in \NN$, 	the surrogate loss function $\phi(t)=\log (1+e^{-t})$ be the logistic loss. Under the same setting of \Cref{theorem3}, suppose that $\eta \in W^\alpha_\infty (\mathcal{X})$, and the radius of adversarial training $\delta< \frac{q_X}{3} \leq \frac{1}{3} n^{-\frac{1}{d}}$. Then $\forall C_0 \in (0,1]$, there exist infinity many adversarial training global minima $\widehat{f}_D^{over}$ with depth $L= \mathcal{O} \left(  \log \frac{1}{\zeta} \right) $, width $d_1=  \mathcal{O} \left(  \zeta ^{-\frac{d}{\alpha}} \log \frac{1}{\zeta}+ n \right) $, $d_2, \dots, d_L= \mathcal{O} \left( \zeta ^{-\frac{d}{\alpha}} \log \frac{1}{\zeta} \right) $, and non-zero free parameters $ \mathcal{O} \left( \zeta ^{-\frac{d}{\alpha}} \log \frac{1}{\zeta}+ n \right) $, such that the adversarial training error $\widehat{\mathcal{E}_D^{\phi,\delta}}\left( \widehat{f}_D^{over}\right)$ can be arbitrarily small, and
		\begin{equation}
			\mathbb{E} \left[ \mathcal{R}\left( \text{sgn}\left( \widehat{f}_D^{over}\right) \right) -\mathcal{R}(f_c) \right]  \leq 2 \|J_\rho\|  \left( (2+2C_0) \delta\right) ^d n,
		\end{equation}
		\begin{equation}
			\mathbb{E} \left[ \mathcal{R}^\delta \left( \text{sgn}\left(\widehat{f}_D^{over}\right) \right)  - \mathcal{R}^\delta (f_c^\delta) \right]  \leq 3 \|J_\rho\|  \left( (4+2C_0) \delta\right) ^d n,
		\end{equation}
	\end{theorem}
 
	\begin{proof} [Proof of \Cref{theorem4}]
		The idea is to construct the adversarial training global minima obtained by the logistic loss $\phi_{LR}$ based on the construction of those obtained by the hinge loss $\phi$ in the proof of  \Cref{theorem3}, i.e.,  $f_{D,\theta,\delta,\tau,\epsilon}^{\phi}(\bm x)$ in \eqref{fDLS}. For $\beta \in (0,1)$ that can be arbitrarily small, denote
		\begin{equation}
			f_{D,\theta,\delta,\tau,\epsilon}^{\phi_{LR}}(\bm x) = \log \frac{1}{\beta} f_{D,\theta,\delta,\tau,\epsilon}^{\phi}(\bm x).
		\end{equation}
		Since $f_{D,\theta,\delta,\tau,\epsilon}^{\phi} (\bm x) = y_i$, when $\bm x \in B_{\delta,\infty}(\bm x_i)$ as shown in \eqref{fDLSproperty}, we have
		\begin{equation}
			\begin{aligned}
				\widehat{\mathcal{E}_D^{\phi_{LR},\delta}} \left( f_{D,\theta,\delta,\tau,\epsilon}^{\phi_{LR}}\right) &=\frac{1}{n} \sum_{i=1}^n  \max_{\bm x'_i \in B_{\delta,\infty}(\bm x_i)}\log \left(1+ e^{- y_i  \log \frac{1}{\beta} f_{D,\theta,\delta,\tau,\epsilon}^{\phi}(\bm x'_i) } \right) \\
				&= \log(1+ \beta)  \leq \beta.
			\end{aligned}
		\end{equation}
		Moreover, since $\log \frac{1}{\beta}>0$, $f_{D,\theta,\delta,\tau,\epsilon}^{\phi_{LR}}(\bm x)$ has the same sign with $f_{D,\theta,\delta,\tau,\epsilon}^{\phi}(\bm x)$ for all $\bm x \in \mathcal{X}$. This indicates that $\mathcal{R}^\delta \left( \text{sgn}\left(f_{D,\theta,\delta,\tau,\epsilon}^{\phi_{LR}}\right) \right) = \mathcal{R}^\delta \left( \text{sgn}\left(f_{D,\theta,\delta,\tau,\epsilon}^{\phi}\right) \right) $, and the adversarial misclassification error bound will be the same as in \Cref{theorem3}. Thus we complete the proof.
	\end{proof}

	\section{Proof of Main Results in \Cref{section3}}
	
	\subsection{Proof of \Cref{theorem1}}
	The proof of \Cref{theorem1} also follows the teacher-student network scheme. The ERM estimators on under-parameterized deep ReLU FNNs that possess good generalization performance are considered to be the teacher network, and we construct the student network by deepening the teacher network to ensure that it achieves zero adversarial training error while still maintaining good generalization performance. 
	We prove \Cref{theorem1} based on \Cref{lemma1} and \Cref{lemma2}.
	
	\begin{proof} [Proof of \Cref{theorem1}]
		Let $f_D^{under}= \pi_M \argmin_{f\in \mathcal{F}_{\vec{d},L}} \widehat{\mathcal{E}}_D (f)$ be the truncated ERM estimator on the under-parameterized deep ReLU FNNs in \Cref{lemma1} with $L\sim \log n$, $d_1 \sim n^{\frac{d}{2\alpha+d}}$, and $d_2, d_3, \dots d_L \sim \log n$. By \Cref{lemma1}, we have
		\begin{equation} \label{error1}
			\sup_{f_\rho \in W_\infty^\alpha (\mathcal{X}), \rho_X \in \Phi_\rho} \mathbb{E}\left[ \left \|f_D^{under}- f_\rho \right\|_{\rho}^2 \right]  \leq C_1 \left(\frac{n}{\log n} \right) ^{-\frac{2\alpha}{2\alpha+d}}.
		\end{equation}
		Taking $f_D^{under}$ as the teacher network, we then construct the student network based on $f_D^{under}$. 
		Denote $c_1= \left\| f_D^{under}\right\|_{L^\infty(\mathcal{X})}\leq M$, define the student network 
		\begin{equation}
			f_{D,\delta,\tau,\epsilon}^{student}(\bm x):= \sum_{i=1}^n y_i  \Gamma_{\bm x_i-\delta,\bm x_i+\delta,\tau}(\bm x ) + c_1 \tilde \times_\epsilon \left(\frac{f_D^{under}(\bm x)}{c_1}, 1- \sum_{i=1}^n  \Gamma_{\bm x_i-\delta,\bm x_i+\delta,\tau}(\bm x ) \right).
		\end{equation}
		By writing $f_{D,\delta,\tau,\epsilon}^{student}$ as
		\begin{equation*}
			\begin{aligned}
				f_{D,\delta,\tau,\epsilon}^{student}(\bm x) &= \sigma \left( \sum_{i=1}^n y_i  \Gamma_{\bm x_i-\delta,\bm x_i+\delta,\tau}(\bm x ) \right) - \sigma \left( -\sum_{i=1}^n y_i  \Gamma_{\bm x_i-\delta,\bm x_i+\delta,\tau}(\bm x ) \right) \\ 
				&+ c_1 \tilde \times_\epsilon \left(\frac{f_D^{under}(\bm x)}{c_1}, 1- \sum_{i=1}^n  \Gamma_{\bm x_i-\delta,\bm x_i+\delta,\tau}(\bm x ) \right).
			\end{aligned}
		\end{equation*}
		This is in fact a deep ReLU FNN with depth $L= \mathcal{O} \left( \log n+ \log \frac{1}{\epsilon}\right) $, and width $d_1=  \mathcal{O} \left( n^{\frac{d}{2\alpha+d}}+ n + \log \frac{1}{\epsilon}\right) $, and $d_2,\dots, d_L= \mathcal{O} \left( \log n+ \log \frac{1}{\epsilon}\right)$.
		Notice from \eqref{Gammafun} that
		\begin{equation*}
			\Gamma_{\bm x_i-\delta,\bm x_i+\delta,\tau}(\bm x)= \left\{\begin{aligned}
				0, \quad &\hbox{if}\ \bm x \notin[\bm x_i-\delta-\tau,\bm x_i+\delta+\tau]^d, \\
				1, \quad &\hbox{if}\ \bm x \in [\bm x_i-\delta,\bm x_i+\delta]^d.
			\end{aligned}\right.
		\end{equation*}
		When $\bm x\in [\bm x_i-\delta,\bm x_i+\delta]^d$, we have
		$\Gamma_{\bm x_i-\delta,\bm x_i+\delta,\tau}(\bm x)= 1$. Moreover, choose $\tau \leq \delta< \frac{q_X}{3}$, since $2\delta+ \tau < q_X$, we further have $\Gamma_{\bm x_j-\delta,\bm x_j+\delta,\tau}(\bm x)= 0$ for all $j \neq i$. Thus $1- \sum_{i=1}^n  \Gamma_{\bm x_i-\delta,\bm x_i+\delta,\tau}(\bm x ) = 0$, then we get
		$\tilde \times_\epsilon \left(\frac{f_D^{under}(\bm x)}{c_1}, 1- \sum_{i=1}^n  \Gamma_{\bm x_i-\delta,\bm x_i+\delta,\tau}(\bm x ) \right)=0$ by \Cref{lemma2}. Therefore,
		\begin{equation}
			f_{D,\delta,\tau,\epsilon}^{student} (\bm x) = y_i, \quad \hbox{when} \ \bm x\in [\bm x_i-\delta,\bm x_i+\delta]^d.
		\end{equation}
		This suggests that $\widehat{\mathcal{E}}_D\left( f_{D,\delta,\tau,\epsilon}^{student}\right) =0$, and $f_{D,\delta,\tau,\epsilon}^{student}$ is indeed the global minimum of the adversarial training. What remains to show is that $f_{D,\delta,\tau,\epsilon}^{student}$ achieves good standard generalization performance as $f_D^{under}$, i.e., the distance between the student network $f_{D,\delta,\tau,\epsilon}^{student}$ and the teacher network $f_D^{under}$ is small. Denote the intermediate term for error decomposition
		\begin{equation}
			f_{D,\delta,\tau}(\bm x):= \sum_{i=1}^n y_i  \Gamma_{\bm x_i-\delta,\bm x_i+\delta,\tau}(\bm x ) + f_D^{under}(\bm x) \left( 1- \sum_{i=1}^n  \Gamma_{\bm x_i-\delta,\bm x_i+\delta,\tau}(\bm x ) \right).
		\end{equation}
		Since $\delta< \frac{q_X}{3}$, we have $1- \sum_{i=1}^n  \Gamma_{\bm x_i-\delta,\bm x_i+\delta,\tau}(\bm x ) \in [0,1]$, then by \Cref{lemma2}, we get
		\begin{equation} 
			\left\| \left( f_{D,\delta,\tau,\epsilon}^{student}- f_{D,\delta,\tau}\right) ^2 \right\|_{L^1(\mathcal{X})} \leq \left\| \left( f_{D,\delta,\tau,\epsilon}^{student}- f_{D,\delta,\tau} \right) ^2 \right\|_{L^\infty(\mathcal{X})} \leq c_1^2 \epsilon^2.
		\end{equation}
		Notice that when $\bm x\in \mathcal{X} \backslash \left( \cup_{i\in \{1,\dots,n\}} [\bm x_i-\delta-\tau,\bm x_i+\delta+\tau]^d \right) $,  $\Gamma_{\bm x_i-\delta,\bm x_i+\delta,\tau}(\bm x)=0$ for all $i$, we have $f_{D,\delta,\tau}(\bm x)=  f_D^{under}(\bm x)$. It follows from $\tau \leq \delta$ that
		\begin{equation}
			\begin{aligned}
				\left\| \left(  f_{D,\delta,\tau}-  f_D^{under} \right) ^2 \right\|_{L^1(\mathcal{X})} &= \sum_{i=1}^n \int_{ [\bm x_i-\delta-\tau,\bm x_i+\delta+\tau]^d} \left( f_{D,\delta,\tau}(\bm x)-  f_D^{under}(\bm x) \right)^2 d \bm x \\
				& \leq \left( c_1+M\right)^2  4^d \delta^d n.
			\end{aligned}
		\end{equation}
		By the assumption that $\|J_\rho\|< \infty$, we get
		\begin{equation} \label{A11}
			\begin{aligned}
				& \left\| f_{D,\delta,\tau,\epsilon}^{student}- f_D^{under}  \right\|_{\rho}^2 = \int_\mathcal{X} \left( f_{D,\delta,\tau,\epsilon}^{student}- f_D^{under}\right) ^2 d\rho_X \\
				\leq & 2  \left\| \left( f_{D,\delta,\tau,\epsilon}^{student}- f_{D,\delta,\tau} \right) ^2 \right\|_{L_{\rho_X}^1(\mathcal{X})} +  2 \left\| \left(  f_{D,\delta,\tau}- f_D^{under}\right) ^2 \right\|_{L_{\rho_X}^1(\mathcal{X})} \\
				\leq & 2 \|J_\rho\| \left\| \left( f_{D,\delta,\tau,\epsilon}^{student}- f_{D,\delta,\tau} \right) ^2 \right\|_{L^1(\mathcal{X})} +  2 \|J_\rho\| \left\| \left(  f_{D,\delta,\tau}- f_D^{under}\right) ^2 \right\|_{L^1(\mathcal{X})} \\
				\leq & 
				2  \|J_\rho\| \left( c_1^2 \epsilon^2 + \left( c_1+M\right)^2  4^d \delta^d n \right).
			\end{aligned}
		\end{equation}
		By choosing $\epsilon= n^{-\frac{\alpha}{2\alpha+d}}$, and since $\delta< \min \left\{\frac{q_X}{3}, n^{-\frac{2\alpha}{(2\alpha+d)d}-\frac{1}{d}} \right\}$, we have
		\begin{equation}
			\left\| f_{D,\delta,\tau,\epsilon}^{student}- f_D^{under}  \right\|_{\rho}^2 \leq   \|J_\rho\| c_2 n^{-\frac{2\alpha}{2\alpha+d}},
		\end{equation}
		where $c_2= 2c_1^2 + 2(c_1+M)^2 4^d$.
		Combining with \eqref{error1}, finally we obtain
		\begin{equation}
			\begin{aligned}
				& \sup_{f_\rho \in W_\infty^\alpha (\mathcal{X}), \rho_X \in \Phi_\rho} \mathbb{E}\left[ \mathcal{E}\left( f_{D,\delta,\tau,\epsilon}^{student} \right) - \mathcal{E}\left( f_\rho\right)  \right] \\
				= & \sup_{f_\rho \in W_\infty^\alpha (\mathcal{X}), \rho_X \in \Phi_\rho} \mathbb{E}\left[ \left \|f_{D,\delta,\tau,\epsilon}^{student}- f_\rho \right\|_{\rho}^2 \right]  \\ 
				\leq & \sup_{f_\rho \in W_\infty^\alpha (\mathcal{X}), \rho_X \in \Phi_\rho} 2\mathbb{E}\left[ \left \|f_{D,\delta,\tau,\epsilon}^{student}- f_D^{under} \right\|_{\rho}^2 \right] \\
				+ & \sup_{f_\rho \in W_\infty^\alpha (\mathcal{X}), \rho_X \in \Phi_\rho} 2\mathbb{E}\left[ \left \|f_D^{under}- f_\rho \right\|_{\rho}^2 \right] \\
				\leq & C_2 \left(\frac{n}{\log n} \right) ^{-\frac{2\alpha}{2\alpha+d}},
			\end{aligned}
		\end{equation}
		where $C_2= 2c_2 \|J_\rho\|+ 2C_1^2$. Moreover, since $\tau \leq \delta$ can be arbitrarily chosen, we conclude that there are infinitely many $f_{D,\delta,\tau,\epsilon}^{student} \in \Delta_{\delta,\vec d, L}$ that can achieve near-optimal rates of convergence for standard generalization error, where the depth $L= \mathcal{O} \left( \log n\right) $, and width $d_1= \mathcal{O} \left( n\right)$, $d_2,\dots, d_L= \mathcal{O} \left( \log n\right)$. This completes the proof.
	\end{proof}
	
	\subsection{Proof of \Cref{theorem2}} \label{appendixB2}
	
	To prove \Cref{theorem2}, we need the following lemma from \cite[Theorem 4.1]{guhring2020error} which describes approximation rates of deep ReLU FNNs for Sobolev functions with respect to weaker Sobolev norms.
	\begin{lemma} \label{lemma3}
		Let $\alpha \geq 2$ be an integer, $B>0$, and $0 \leq s \leq 1$. Suppose that $f \in W^\alpha_\infty([0,1]^d)$ with $\|f\|_{W^\alpha_\infty([0,1]^d} \leq B$. Then there exists a deep ReLU FNN $\widehat f$ with depth $L=  \mathcal{O}\left( \log \frac{1}{\epsilon}\right) $, width $d_1, \dots, d_L=   \mathcal{O}\left( \epsilon^{-\frac{d}{\alpha-s}} \log \frac{1}{\epsilon}\right) $,  and non-zero free parameters $  \mathcal{O}\left( \epsilon^{-\frac{d}{\alpha-s}} \log \frac{1}{\epsilon}\right)$, such that
		\begin{equation}
			\left\| \widehat f - f \right\|_{W^s_\infty([0,1]^d)} \leq \epsilon.
		\end{equation}
	\end{lemma}
	
	We are now ready to prove \Cref{theorem2} based on \Cref{proposition1}, \Cref{lemma3}, and the proof of \Cref{theorem1}.
	\begin{proof} [Proof of \Cref{theorem2}]
		By \Cref{lemma3}, choose $s=1$. There exists a deep ReLU FNN $f_\theta$ with depth $L=   \mathcal{O}\left( \log \frac{1}{\theta}\right) $, width $d_1, \dots, d_L=   \mathcal{O}\left( \theta^{-\frac{d}{\alpha-1}} \log \frac{1}{\theta}\right) $,  and non-zero free parameters $  \mathcal{O}\left( \theta^{-\frac{d}{\alpha-1}} \log \frac{1}{\theta}\right)$, such that
		\begin{equation} \label{errorA16}
			\left\| f_\theta- f_\rho \right\|_{W^1_\infty(\mathcal{X})} \leq \theta.
		\end{equation}
		Denote $c_3= \left\| f_\theta \right\|_{L^\infty(\mathcal{X})} \leq B+1$. Similar as the proof in \Cref{theorem1}, we use $f_\theta$ as the teacher network, and construct the student network $f_{D,\theta,\delta,\tau,\epsilon}^{student}$ which is the adversarial training global minimum
		\begin{equation}
			f_{D,\theta,\delta,\tau,\epsilon}^{student}(\bm x):= \sum_{i=1}^n y_i  \Gamma_{\bm x_i-\delta,\bm x_i+\delta,\tau}(\bm x ) + c_3 \tilde \times_\epsilon \left(\frac{f_\theta(\bm x)}{c_3}, 1- \sum_{i=1}^n  \Gamma_{\bm x_i-\delta,\bm x_i+\delta,\tau}(\bm x ) \right).
		\end{equation}
		Choose $\tau \leq C_0\delta< \frac{C_0 q_X}{3}$. Same as the proof of \Cref{theorem1}, we have the property that
		\begin{equation}
			f_{D,\theta,\delta,\tau,\epsilon}^{student} (\bm x) = y_i, \quad \hbox{when} \ \bm x\in [\bm x_i-\delta,\bm x_i+\delta]^d.
		\end{equation}
		Therefore, $\widehat{\mathcal{E}}_D\left( f_{D,\theta,\delta,\tau,\epsilon}^{student}\right) =0$, and $f_{D,\theta,\delta,\tau,\epsilon}^{student}$ is indeed the global minimum of the adversarial training. Moreover, $f_{D,\theta,\delta,\tau,\epsilon}^{student}$ is a deep ReLU FNN with depth $L= \mathcal{O} \left(  \log \frac{1}{\theta} + \log \frac{1}{\epsilon}\right) $, width $d_1=  \mathcal{O} \left( \theta^{-\frac{d}{\alpha-1}} \log \frac{1}{\theta}+ n + \log \frac{1}{\epsilon}\right) $, $d_2, \dots, d_L= \mathcal{O} \left( \theta^{-\frac{d}{\alpha-1}} \log \frac{1}{\theta} + \log \frac{1}{\epsilon}\right) $, and non-zero free parameters $ \mathcal{O}\left( \theta^{-\frac{d}{\alpha-1}} \log \frac{1}{\theta}+  n + \log \frac{1}{\epsilon} \right)$.
		
		We then bound the adversarial generalization error of this adversarial training global minimum.
		By \Cref{proposition1}, we only need to bound two error terms: $\mathcal{E}^\delta \left(f_{D,\theta,\delta,\tau,\epsilon}^{student}\right)  - \mathcal{E} \left( f_{D,\theta,\delta,\tau,\epsilon}^{student}\right) $ and $ \mathcal{E} \left( f_{D,\theta,\delta,\tau,\epsilon}^{student}\right)  - \mathcal{E} (f_\rho)$. We first consider the second error term, since the construction of the student network $f_{D,\theta,\delta,\tau,\epsilon}^{student}$ from the teacher network $f_\theta$ is the same as in the proof of \Cref{theorem1}, from \eqref{A11} we have
		\begin{equation*} 
			\left\| f_{D,\theta,\delta,\tau,\epsilon}^{student}- f_\theta  \right\|_{\rho}^2  \leq 2  \|J_\rho\| \left( c_3^2 \epsilon^2 + \left( c_3+M\right)^2  (4 \delta)^d n \right).
		\end{equation*}
		Moreover, from \eqref{errorA16}, we have
		\begin{equation*}
			\left\| f_\theta- f_\rho  \right\|_{\rho}^2 \leq \|J_\rho\| \left\| \left( f_\theta- f_\rho \right) ^2 \right\|_{L^1(\mathcal{X})}  \leq \|J_\rho\| \left\| \left( f_\theta- f_\rho\right) ^2 \right\|_{L^\infty(\mathcal{X})} \leq \|J_\rho\| \theta^2.
		\end{equation*}
		It follows that
		\begin{equation} \label{A20}
			\begin{aligned}
				\mathcal{E}\left( f_{D,\theta,\delta,\tau,\epsilon}^{student}\right)   - \mathcal{E}\left( f_\rho\right)  &=  \left \|f_{D,\theta,\delta,\tau,\epsilon}^{student}- f_\rho \right\|_{\rho}^2  \\ 
				& \leq  2 \left \|f_{D,\theta,\delta,\tau,\epsilon}^{student}- f_\theta \right\|_{\rho}^2+  2 \left \|f_\theta- f_\rho \right\|_{\rho}^2 \\
				& \leq 2 \|J_\rho\| \left( 2c_3^2 \epsilon^2 + 2\left( c_3+M\right)^2  (4 \delta)^d n +\theta^2 \right) ,
			\end{aligned}
		\end{equation}
		We then bound the first error term, denote $c_4= \left\| f_{D,\theta,\delta,\tau,\epsilon}^{student} \right\|_{L^\infty(\mathcal{X})} \leq 2c_3+M$,
		\begin{equation*}
			\begin{aligned}
				& \mathcal{E}^\delta  \left(f_{D,\theta,\delta,\tau,\epsilon}^{student}\right)  - \mathcal{E} \left( f_{D,\theta,\delta,\tau,\epsilon}^{student}\right) \\
				= & \int_{\mathcal{Z}} \max_{\bm x' \in B_{\delta,\infty}(\bm x)} \left(f_{D,\theta,\delta,\tau,\epsilon}^{student}(\bm x')-y \right)^2 - \left(f_{D,\theta,\delta,\tau,\epsilon}^{student}(\bm x)-y \right)^2 d\rho \\
				\leq & (2c_4+2M) \int_{\mathcal{X}} \max_{\bm x' \in B_{\delta,\infty}(\bm x)} \left|f_{D,\theta,\delta,\tau,\epsilon}^{student}(\bm x')- f_{D,\theta,\delta,\tau,\epsilon}^{student}(\bm x) \right| d\rho_X \\
				\leq & (2c_4+2M) \|J_\rho\| \int_{\mathcal{X}} \max_{\bm x' \in B_{\delta,\infty}(\bm x)} \left|f_{D,\theta,\delta,\tau,\epsilon}^{student}(\bm x')- f_{D,\theta,\delta,\tau,\epsilon}^{student}(\bm x) \right| d\bm x.
			\end{aligned}
		\end{equation*}
		To bound this term, we divide $\mathcal{X}$ to two disjoint parts: $\cup_{i\in \{1,\dots,n\}} [\bm x_i-2\delta-\tau,\bm x_i+2\delta+\tau]^d$ and $\mathcal{X} \backslash \cup_{i\in \{1,\dots,n\}} [\bm x_i-2\delta-\tau,\bm x_i+2\delta+\tau]^d$. We first consider the second part, for any $\bm x$ belongs to the second part, we have $f_{D,\theta,\delta,\tau,\epsilon}^{student}(\bm x')= f_\theta(\bm x')$ for any $\bm x' \in B_{\delta,\infty}(\bm x)$. Moreover, by \eqref{errorA16}, the derivative of $f_\theta- f_\rho$ is bounded by $\theta$, thus
		\begin{equation*}
			\|f_\theta \|_{Lip} \leq \|f_\rho \|_{Lip} + \|f_\theta-f_\rho \|_{Lip} \leq B+ \theta.
		\end{equation*}
		Therefore, we have
		\begin{equation*}
			\begin{aligned}
				& \int_{\mathcal{X} \backslash \cup_{i\in \{1,\dots,n\}} [\bm x_i-2\delta-\tau,\bm x_i+2\delta+\tau]^d} \max_{\bm x' \in B_{\delta,\infty}(\bm x)} \left|f_{D,\theta,\delta,\tau,\epsilon}^{student}(\bm x')- f_{D,\theta,\delta,\tau,\epsilon}^{student}(\bm x) \right|d\bm x \\
				= & \int_{\mathcal{X} \backslash \cup_{i\in \{1,\dots,n\}} [\bm x_i-2\delta-\tau,\bm x_i+2\delta+\tau]^d}  \max_{\bm x' \in B_{\delta,\infty}(\bm x)} \left|f_\theta(\bm x')- f_\theta(\bm x) \right| d\bm x \\
				\leq & \int_{\mathcal{X} \backslash \cup_{i\in \{1,\dots,n\}} [\bm x_i-2\delta-\tau,\bm x_i+2\delta+\tau]^d}  \max_{\bm x' \in B_{\delta,\infty}(\bm x)} \|f_\theta \|_{Lip} \|\bm x-\bm x'\|_2   d\bm x  \\
				\leq &  (B+\theta) \sqrt{d} \delta.
			\end{aligned}
		\end{equation*}
		For the first part, we have
		\begin{equation*}
			\begin{aligned}
				& \sum_{i=1}^n \int_{ [\bm x_i-2\delta-\tau,\bm x_i+2\delta+\tau]^d}  \max_{\bm x' \in B_{\delta,\infty}(\bm x)} \left|f_{D,\theta,\delta,\tau,\epsilon}^{student}(\bm x')- f_{D,\theta,\delta,\tau,\epsilon}^{student}(\bm x) \right| d\bm x \\
				\leq & 2c_4  \left( (4+2C_0) \delta\right) ^d n.
			\end{aligned}
		\end{equation*}
		Combining these two terms,  the second error term can be bounded by
		\begin{equation*}
			\mathcal{E}^\delta  \left(f_{D,\theta,\delta,\tau,\epsilon}^{student}\right)  - \mathcal{E} \left( f_{D,\theta,\delta,\tau,\epsilon}^{student}\right) \leq (2c_4+2M)\|J_\rho\| \left( (B+\theta) \sqrt{d} \delta+  2c_4 \left( (4+2C_0) \delta\right) ^d n \right).
		\end{equation*}
		Finally, by \Cref{proposition1}, we get
		\begin{equation*}
			\begin{aligned}
				\mathcal{E}^\delta (f_{D,\theta,\delta,\tau,\epsilon}^{student}) - \mathcal{E}^\delta (f_\rho^\delta) & \leq \mathcal{E}^\delta \left(f_{D,\theta,\delta,\tau,\epsilon}^{student}\right)  - \mathcal{E} \left( f_{D,\theta,\delta,\tau,\epsilon}^{student}\right) + \mathcal{E} \left( f_{D,\theta,\delta,\tau,\epsilon}^{student}\right)  - \mathcal{E} (f_\rho) \\
				& \leq 2 \|J_\rho\| \left( 2c_3^2 \epsilon^2 + 2\left( c_3+M\right)^2  4^d \delta^d n +\theta^2 \right)  \\
				& + (2c_4+2M)\|J_\rho\| \left( (B+\theta) \sqrt{d} \delta+ 2c_4 \left( (4+2C_0) \delta\right) ^d n \right).
			\end{aligned}
		\end{equation*}
		By choosing $\epsilon= \theta= \sqrt{\delta}$, $f_{D,\theta,\delta,\tau,\epsilon}^{student}$ is a deep ReLU FNN with depth $L= \mathcal{O} \left(  \log \frac{1}{\delta} \right) $, width $d_1=  \mathcal{O} \left( \delta^{-\frac{d}{2\alpha-2}} \log \frac{1}{\delta}+ n \right) $, $d_2, \dots, d_L= \mathcal{O} \left( \delta^{-\frac{d}{2\alpha-2}} \log \frac{1}{\delta} \right) $, and non-zero free parameters $ \mathcal{O} \left( \delta^{-\frac{d}{2\alpha-2}} \log \frac{1}{\delta}+ n \right) $. Furthermore, the adversarial generalization error bound is
		\begin{equation}
			\mathbb{E}\left[ \mathcal{E}^\delta (f_{D,\theta,\delta,\tau,\epsilon}^{student}) - \mathcal{E}^\delta (f_\rho^\delta) \right]  \leq C_3 \sqrt{d} \max \left\{ \delta, \left( (4+2C_0) \delta\right) ^d n \right\} ,
		\end{equation}
		where $C_3=  2\|J_\rho\|  \left( 2c_3^2+ 2\left( c_3+M\right)^2 +1 \right)  + (2c_4+2M)\|J_\rho\| \left( B+1+ 2c_4 \right)$. Moreover, when $n^{-{\frac{1}{d-1}}} \leq \delta< \frac{q_X}{3} \leq \frac{1}{3} n^{-\frac{1}{d}}$, we have
		\begin{equation}
			\mathbb{E}\left[ \mathcal{E}^\delta (f_{D,\theta,\delta,\tau,\epsilon}^{student}) - \mathcal{E}^\delta (f_\rho^\delta) \right]  \leq
			C_3 \sqrt{d}  \left( (4+2C_0) \delta\right) ^d n.
		\end{equation}
		Moreover, since $\tau \leq \delta$ can be arbitrarily chosen, we conclude that there are infinitely many $f_{D,\theta,\delta,\tau,\epsilon}^{student} \in \Delta_{\delta,\vec d, L}$ that can achieve such adversarial generalization error bound. Thus we complete the proof.
	\end{proof}

	\subsection{Proof of \Cref{lowerbound1}}
	\begin{proof} [Proof of \Cref{lowerbound1}]
		For  any $\bm x \in [\bm x_i-2\delta,\bm x_i+2\delta]^d$, notice that $B_{\delta,\infty}(\bm x) \cap [\bm x_i-\delta,\bm x_i+\delta]^d  \neq \O$. Moreover, for any global minimum of adversarial training $\widehat{f}_D^{over} \in \Delta_{\delta, \vec d, L}$ with non-zero free parameters $ \mathcal{O} \left( \delta^{-\frac{d}{2\alpha-2}} \log \frac{1}{\delta}+ n \right) $, since it can achieve zero adversarial training error as is constructed in \Cref{appendixB2},  when $\bm x \in [\bm x_i-\delta,\bm x_i+\delta]^d$, we always have $\widehat f_{D}^{over}(\bm x)= y_i$. Therefore,
		\begin{equation*}
			\begin{aligned}
				& \int_{\bm x \in [\bm x_i-2\delta,\bm x_i+2\delta]^d} \int_Y \max_{\bm x' \in B_{\delta,\infty}(\bm x)} \left(  \widehat f_{D}^{over}(\bm x') -y \right) ^2  d\rho \\
				\geq &  \int_{
					\bm x \in [\bm x_i-2\delta,\bm x_i+2\delta]^d} \int_Y  \left( y_i -y \right) ^2  d\rho,
			\end{aligned}
		\end{equation*}
		it follows that
		\begin{equation*}
			\begin{aligned}
				\mathcal{E}^\delta (\widehat{f}_D^{over}) - \mathcal{E} \left( f_\rho\right) & \geq \int_{\bm x \in \cup_{i \in \{1,\dots,n\}} [\bm x_i-2\delta,\bm x_i+2\delta]^d} \int_Y \left( y_i -y \right) ^2 - \left(f_\rho(\bm x)- y \right)^2   d\rho \\
				& =  \int_{\bm x \in \cup_{i \in \{1,\dots,n\}} [\bm x_i-2\delta,\bm x_i+2\delta]^d}  \left(f_\rho(\bm x)- y_i \right)^2   d\rho_X  ,
			\end{aligned}
		\end{equation*}
		thus
		\begin{equation*}
			\begin{aligned}
				& \mathbb E \left[ \mathcal{E}^\delta (\widehat{f}_D^{over}) - \mathcal{E} \left( f_\rho\right) \right]  \\
				\geq & \mathbb{E} \left[\sum_{i=1}^n \int_{ \bm x \in [\bm x_i-2\delta,\bm x_i+2\delta]^d}   \left(f_\rho(\bm x)- y_i \right)^2   d\rho_X \right]  \\
				=& \sum_{i=1}^n \int_{\bm x_i} \int_{ \bm x \in [\bm x_i-2\delta,\bm x_i+2\delta]^d} \int_{y_i} \left(f_\rho(\bm x)- f_\rho(\bm x_i) \right)^2 + \left(f_\rho(\bm x_i)- y_i \right)^2 d\rho(y_i|\bm x_i) d \rho_X d \rho_X \\
				\geq & \sum_{i=1}^n \int_{\bm x_i} \int_{ \bm x \in [\bm x_i-2\delta,\bm x_i+2\delta]^d} \sigma^2 d \rho_X d \rho_X \\
				\geq &  \| \bar J_\rho \| \sigma^2 (4 \delta)^d n.
			\end{aligned}
		\end{equation*}
		Moreover, notice that
		\begin{equation*}
			\begin{aligned}
				& \mathcal{E}^\delta  \left(f_\rho \right)  - \mathcal{E} \left( f_\rho \right) \\
				= & \int_{\mathcal{Z}} \max_{\bm x' \in B_{\delta,\infty}(\bm x)} \left(f_\rho (\bm x')-y \right)^2 - \left(f_\rho (\bm x)-y \right)^2 d\rho \\
				\leq & (2B+2M) \int_{\mathcal{X}} \max_{\bm x' \in B_{\delta,\infty}(\bm x)} \left|f_\rho (\bm x')- f_\rho (\bm x) \right| d\rho_X \\
				\leq & (2B+2M) \|J_\rho\| \int_{\mathcal{X}} \max_{\bm x' \in B_{\delta,\infty}(\bm x)} \left|f_\rho (\bm x')- f_\rho (\bm x) \right| d\bm x \\
				\leq & (2B+2M) \|J_\rho\| B \sqrt{d} \delta.
			\end{aligned}
		\end{equation*}
		Therefore, take $\bar C_1 = (2B+2M)  B $, and use the fact that $\mathcal{E}^\delta (f_\rho^\delta) \leq \mathcal{E}^\delta (f_\rho)$, we have
		\begin{equation*}
			\begin{aligned}
				\mathbb{E} \left[ \mathcal{E}^\delta (\widehat{f}_D^{over}) - \mathcal{E}^\delta (f_\rho^\delta) \right] 
				&=  \mathbb E \left[ \mathcal{E}^\delta (\widehat{f}_D^{over}) - \mathcal{E} \left( f_\rho\right) \right]  - \mathbb E \left[\mathcal{E}^\delta (f_\rho^\delta)- \mathcal{E}\left( f_\rho \right)  \right] \\
				&\geq  \bar C_1 \sigma^2 \delta^d n - \left[\mathcal{E}^\delta (f_\rho)- \mathcal{E}\left( f_\rho \right)  \right] \\
				& \geq \| \bar J_\rho \|  \sigma^2 (4\delta)^d n  - \bar C_1 \|J_\rho\| \sqrt{d} \delta.
			\end{aligned}
		\end{equation*}
		Thus we complete the proof.
	\end{proof}
	

	\vskip 0.2in
	
	\bibliographystyle{plain}
	\bibliography{robust_reference}
	
\end{document}